\documentclass{article}
\usepackage[utf8]{inputenc}
\usepackage{mathtools}
\title{Hutchinson Trace
Estimation for High-Dimensional and High-Order Physics-Informed Neural Networks}

\usepackage[a4paper,top=3cm,bottom=2cm,left=2.3cm,right=2.3cm,marginparwidth=1.75cm]{geometry}
\usepackage{graphicx}
\usepackage{verbatim}
\usepackage{hyperref}       
\usepackage{url}            
\usepackage{booktabs}       
\usepackage{amsfonts}       
\usepackage{nicefrac}       
\usepackage{microtype}      
\usepackage{amsmath}
\usepackage{multirow}
\usepackage{amsthm}
\usepackage{amssymb}
\newtheorem{theorem}{Theorem}[section]
\newtheorem{remark}{Remark}

\usepackage{subfigure}
\usepackage{times}
\usepackage{color}
\newcommand{\bx}{\boldsymbol{x}}
\newcommand{\bv}{\boldsymbol{v}}

\begin{document}
\author{Zheyuan Hu\thanks{Equal Contribution} \ \thanks{Department of Computer Science, National University of Singapore, Singapore, 119077 (\href{mailto:e0792494@u.nus.edu}{e0792494@u.nus.edu},\href{mailto:kenji@nus.edu.sg}{kenji@nus.edu.sg})} \and Zekun Shi\footnotemark[1] \ \footnotemark[2] \and  George Em Karniadakis\thanks{Division of Applied Mathematics, Brown University, Providence, RI 02912, USA (\href{mailto:george\_karniadakis@brown.edu}{george\_karniadakis@brown.edu})} \ \thanks{Advanced Computing, Mathematics and Data Division, Pacific Northwest National Laboratory, Richland, WA, United States} \and Kenji Kawaguchi\footnotemark[2]}

\date{}

\maketitle
\begin{abstract}
Physics-Informed Neural Networks (PINNs) have proven effective in solving partial differential equations (PDEs), especially when some data are available by seamlessly blending data and physics. However, extending PINNs to high-dimensional and even high-order PDEs encounters significant challenges due to the computational cost associated with automatic differentiation in the residual loss function calculation. Herein, we address the limitations of PINNs in handling high-dimensional and high-order PDEs by introducing the Hutchinson Trace Estimation (HTE) method.
Starting with the second-order high-dimensional PDEs, which are ubiquitous in scientific computing, HTE is applied to transform the calculation of the entire Hessian matrix into a Hessian vector product (HVP). This approach not only alleviates the computational bottleneck via Taylor-mode automatic differentiation but also significantly reduces memory consumption from the Hessian matrix to an HVP's scalar output. We further showcase HTE's convergence to the original PINN loss and its unbiased behavior under specific conditions.
Comparisons with the Stochastic Dimension Gradient Descent (SDGD) highlight the distinct advantages of HTE, particularly in scenarios with significant variability and variance among dimensions. We further extend the application of HTE to higher-order and higher-dimensional PDEs, specifically addressing the biharmonic equation. By employing tensor-vector products (TVP), HTE efficiently computes the colossal tensor associated with the fourth-order high-dimensional biharmonic equation, saving memory and enabling rapid computation.
The effectiveness of HTE is illustrated through experimental setups, demonstrating comparable convergence rates with SDGD under memory and speed constraints. Additionally, HTE proves valuable in accelerating the Gradient-Enhanced PINN (gPINN) version as well as the Biharmonic equation.  Overall, HTE opens up a new capability in scientific machine learning for tackling high-order and high-dimensional PDEs.
\end{abstract}

\section{Introduction}
Physics-Informed Neural Networks (PINNs) \cite{raissi2018forward} have made significant strides in solving partial differential equation (PDE) problems in scientific computing, particularly for low-dimensional equations. This is attributed to their powerful fitting capability \cite{kawaguchi2016deep}, generalization \cite{kawaguchi2017generalization}, and stability in optimization \cite{kingma2014adam}. However, despite the demonstrated capacity of neural networks in modeling high-dimensional data in image and text domains, PINNs' exploration of high-dimensional and high-order PDEs has been relatively limited. The primary challenge arises from the enormous computational cost involved in the automatic differentiation process of calculating the residual loss function steps for high-dimensional and high-order PDEs under the PINNs framework.

Solving problems related to high-dimensional PDEs is crucial, considering the curse-of-dimensionality, which is prevalent in various fields such as the Black-Scholes equation in mathematical finance for option pricing, the Hamilton-Jacobi-Bellman (HJB) equation in optimal control, and the Schrödinger equation in quantum physics. Despite the availability of specialized high-dimensional PDE solvers for different domains, such as DeepBSDE \cite{han2018solving} and Deep Splitting Method \cite{beck2021deep} for high-dimensional parabolic equations and specific algorithms for the Hamilton-Jacobi equation \cite{darbon2020overcoming,darbon2016algorithms}, a unified and effective approach for solving {\em general} high-dimensional PDEs is lacking.

PINNs is a versatile method capable of handling arbitrary PDEs and being mesh-free, in principle it is not susceptible to the curse-of-dimensionality. Despite recent efforts to use PINNs for solving high-dimensional PDEs, such as Stochastic Dimension Gradient Descent (SDGD) \cite{hu2023tackling} and random smoothing \cite{he2023learning,hu2023bias}, these approaches have their limitations. SDGD attempts to reduce each stochastic gradient's computational cost and memory consumption by sampling dimensions. However, when significant variability and variance among dimensions exist, SDGD's stochastic gradients may suffer from substantial variance, hindering convergence. As for random smoothing methods, they smooth the original neural network with Gaussian noise, impacting the network's expressive power and making it less effective in handling non-smooth PDE solutions.

To address the limitations of PINN in tackling high-dimensional and even high-order PDEs, herein, we propose the use of Hutchinson Trace Estimation (HTE) \cite{hutchinson1989stochastic} to accelerate PINNs and reduce memory cost. HTE can be applied to various common PDEs, such as second-order parabolic equations and the biharmonic equation. Firstly, for second-order parabolic PDEs, the computational bottleneck lies in the need to calculate the entire Hessian, with its computational cost increasing quadratically with the dimensionality. HTE is seamlessly integrated into the PINNs framework to alleviate this central computational bottleneck. Specifically, HTE transforms the calculation of the entire Hessian matrix into a Hessian vector product (HVP). Unlike the full Hessian, which is a large matrix of dimensions $d \times d$, where $d$ is the dimensionality of the PDE problem, the output of HVP is a scalar, significantly reducing memory consumption. Moreover, HVP can efficiently utilize Taylor-mode automatic differentiation \cite{bettencourt2019taylormode} in JAX \cite{jax2018github} for rapid computation, achieving both acceleration and reduced memory usage.

Furthermore, we demonstrate that the loss of PINNs with HTE converges to the original version of PINN loss under certain conditions and is unbiased, providing a solid foundation for its convergence properties. We also discuss the distinctions, similarities, and applicability between HTE and SDGD, the most related methodologies. We present examples illustrating when HTE or SDGD is more suitable and how their variances depend on the PDE problem. After addressing second-order parabolic equations, we explore how HTE is even more efficient in higher-order and higher-dimensional scenarios, motivating the extension of HTE to the biharmonic equation.

Given that the biharmonic equation is a fourth-order high-dimensional equation, attempting to solve it directly with conventional PINNs would lead to a massive tensor of size $d^4$, where $d$ represents the dimensionality of the PDE problem, and 4 denotes the order of the equation. However, by extending the Hessian Vector Product (HVP) in HTE to a tensor-vector product (TVP), we can transform the computation of the entire colossal tensor into a scalar which is the TVP output. This not only saves a considerable amount of memory but also facilitates rapid computation using Taylor-mode differentiation. We also demonstrate that the extended HTE method provides an unbiased estimate for the biharmonic operator, and we analyze the sources of its variance. To emphasize the importance of properly implementing the HVP and TVP in the proposed HTE method in JAX, we present a pseudocode to familiarize readers with HTE and its efficient implementation.
 
Ultimately, we showcase the effectiveness of HTE in various experimental setups. Firstly, we demonstrate that HTE and SDGD exhibit similar convergence rates under the same memory constraints and speed limitations on a second-order nonlinear anisotropic parabolic equation. Notably, they successfully scale PINNs to a very high dimension (100,000 dimensions). Secondly, we illustrate how HTE can accelerate the more accurate Gradient-Enhanced PINN (gPINN). Since gPINN requires additional first-order derivatives of the original PINN residual, the computational load is substantial, but HTE mitigates this burden while keeping the improvement brought by gPINN. Lastly, we highlight the advantages of HTE in the context of higher-order, high-dimensional biharmonic equations.

This paper is arranged as follows. We present related work in Section 2. Then, we introduce the proposed HTE method in Section 3.
We present computational experiments in Section 4 and summarize the paper in Section 5.

\section{Related Work}

\textbf{Hutchinson Trace Estimation}. Hutchinson Trace Estimation (HTE) \cite{hutchinson1989stochastic} is an unbiased estimator of the trace of a matrix, which was first proposed for the influence matrix related to the Laplacian smoothing splines computation. Since its initial publication, progress has been made for improvement. Hutch++ \cite{meyer2021hutch++} reduces the variance of HTE by employing a low-rank approximation methodology, whose complexity is proved to be optimal among all matrix-vector query-based methods for trace estimation. Later, Persson et al. \cite{persson2022improved} adopted the Nyström approximation to improve Hutch++ further; \cite{skorski2021modern} provided a modern analysis of the HTE's error. Roosta et al. \cite{roosta2015improved} proved an improved estimation for the sample efficiency in HTE. In modern deep learning, HTE has also been applied to the Diffusion model \cite{song2021scorebased} to infer the probability density function of the model, where HTE can reduce the computation cost when dealing with the Jacobian of the Diffusion model, which can be a convolutional neural network or even a Transformer \cite{vaswani2017attention}, whose derivative is extremely costly. To the best of our knowledge, we are the first to apply HTE for efficient physics-informed neural networks (PINNs) computations for high-order PDEs and in high-dimensions.

\textbf{Efficient Automatic Differentiation}. To implement HTE efficiently, we adopt the Taylor-mode automatic differentiation (AD) \cite{bettencourt2019taylormode} in JAX \cite{jax2018github}, which is tailored for high-order AD. Indeed, the type of AD affects the speed and memory cost significantly, and the commonly used stacked forward/backward modes of AD are much slower. In addition to the superior Taylor-mode AD, other attempts exist to avoid the costly high-order derivatives. Randomized smoothing \cite{he2023learning,hu2023bias} randomly smooths the model using Gaussian noises for inference, and its derivatives can be simulated via Monte Carlo simulation. The generalization ability of a randomly smoothed model can be interpreted by the information bottleneck theory \cite{kawaguchi2023does}. Randomized AD \cite{oktay2021randomized} incorporated stochastic gradient descent (SGD) to accelerate and avoid the need for full AD.

\textbf{High-Dimensional PDE Solver}.
Several works have considered high-dimensional PDE solvers.
In \cite{wang20222}, the significance of $L^\infty$ loss and adversarial training in addressing high-dimensional Hamilton-Jacobi-Bellman equations was demonstrated.
The separable PINNs approach in \cite{cho2022separable} employs a structure that allows the residual points to be a tensor product of per-dimension points, thereby increasing the batch size. Nevertheless, when confronted with problems surpassing ten dimensions, memory usage becomes a serious bottleneck. 
DeepBSDE \cite{han2018solving, han2017deep} and its extensions \cite{beck2019machine, chan2019machine,henry2017deep,hure2020deep,ji2020three,becker2021solving} are grounded in the classical BSDE interpretation for specific high-dimensional parabolic PDEs, leveraging deep learning models to approximate the unknowns in the BSDE formulation.
The deep splitting method \cite{beck2021deep} unifies the splitting method with deep neural networks.
Chen et al. \cite{chen2021solving} solved forward and inverse problems of Fokker-Planck equations, including Brownian noise and Levy noise in high dimensions using PINNs.
FBSNN \cite{raissi2018forward} established a connection between high-dimensional parabolic PDEs and forward-backward stochastic differential equations, employing deep neural networks to learn the unknown solution.
The multilevel Picard methods \cite{beck2020overcoming,beck2020overcoming_ac,becker2020numerical,hutzenthaler2020overcoming, hutzenthaler2021multilevel} represent a nonlinear extension of Monte Carlo capable of solving parabolic PDEs under specific constraints. The deep Galerkin method (DGM) \cite{sirignano2018dgm} trains neural networks to satisfy the high-dimensional PDE operators and other conditions, where derivatives are estimated via Monte Carlo. In \cite{wang2022tensor, wang2022solving}, tensor neural networks were proposed, adopting a separable structure for efficient numerical integration in solving high-dimensional Schr"{o}dinger equations.
More recently, SDGD \cite{hu2023tackling} emerged as a method designed to sample dimensions in PDEs, aiming to scale up and accelerate high-dimensional PINNs. Randomized smoothing \cite{he2023learning,hu2023bias} adopts the smoothed model with Gaussian noise so that Monte Carlo can simulate its inference and derivatives to combat the curse of dimensionality and avoid the costly automatic differentiation in high-order and high dimensions. HTE is also used in scale up and speed up score-based diffusion models \cite{song2021scorebased,hu2024score}.

\textbf{Physics-Informed Machine Learning}. The methodology developed in this paper is based on the concept of Physics-Informed Machine Learning \cite{karniadakis2021physics}, especially Physics-Informed Neural Networks (PINNs) \cite{raissi2019physics}. PINNs model and approximate the PDE solution by neural networks as surrogate models, which are trained through the boundary and residual losses. These approaches are shown theoretically to discover the underlying solutions governed by the PDEs \cite{hu2021extended,mishra2020estimates,shin2020convergence}. PINNs have been successful in numerous fields in science and engineering \cite{cai2021physics,haghighat2021physics,yang2019adversarial,jagtap2022deep}, and effective PINNs variants have been proposed to deal with different applications and problem settings \cite{hu2022augmented,jagtap2020extended,jagtap2020adaptive,jin2021nsfnets,psaros2022meta}.

\section{Proposed Method}
\subsection{Preliminaries}
This paper focuses on employing Physics-Informed Neural Networks (PINNs) to address high-dimensional and high-order Partial Differential Equation (PDE) problems. Additionally, it introduces the Hutchinson Trace Estimation (HTE) to accelerate PINN and reduce its memory consumption. Below, we provide an introduction to PINNs and HTE.

\textbf{Physics-Informed Neural Networks (PINNs)}.
This paper focuses on employing PINNs \cite{raissi2019physics} solving partial differential equations (PDEs) defined on the domain $\Omega \subset \mathbb{R}^d$ with the boundary/initial condition on $\Gamma$ and the residual condition within $\Omega$:
\begin{equation}\label{eq:PDE}
\begin{aligned}
\mathcal{B}u(\bx)=B(\bx) \ \text{on}\ \Gamma, \qquad
\mathcal{L}u(\bx)=g(\bx) \ \text{in}\ \Omega,
\end{aligned}
\end{equation}

Given the boundary points $\{\bx_{b,i}\}_{i=1}^{n_b} \subset \Gamma$ and the residual points $\{\bx_{r,i}\}_{i=1}^{n_r} \subset \Omega$, PINNs minimize the discrepancy in the residual and on the boundary for the PINN neural network model $u_\theta$ parameterized by $\theta$:
\begin{equation}
\begin{aligned}
\mathcal{L}(\theta) &= \lambda_b \mathcal{L}_b(\theta) + \lambda_r \mathcal{L}_r(\theta)\\
&=\frac{\lambda_b}{n_b}\sum_{i=1}^{n_b} {|\mathcal{B}u_{\theta}(\bx_{b,i})-B(\bx_{b,i})|}^2 + \frac{\lambda_r}{n_r}\sum_{i=1}^{n_r} {|\mathcal{L}u_{\theta}(\bx_{r,i})-g(\bx_{r,i})|}^2,
\end{aligned}
\end{equation}
where $\lambda_b$ and $\lambda_r$ are the boundary and residual loss weights, respectively.

\textbf{Hutchinson Trace Estimation (HTE)}.
The trace of a matrix $A \in \mathbb{R}^{d \times d}$ can be randomly estimated as follows:
\begin{align}
\operatorname{Tr}(A) = \mathbb{E}_{\bv \sim p(\bv)}\left[ \bv^\mathrm{T} A \bv\right],
\end{align}
for all random variable $\bv \in \mathbb{R}^d$ such that $\mathbb{E}_{\bv \sim p(\bv)} [\bv\bv^T] = I$.
Therefore, the trace can be estimated by Monte Carlo:
\begin{equation}
\operatorname{Tr}(A)\approx \frac{1}{V}\sum_{i=1}^V \bv_i^\mathrm{T} A\bv_i,
\end{equation}
where each $\bv_i\in\mathbb{R}^d$ are $i.i.d.$ samples from $p(\bv)$.

There are several viable choices for the distribution $p(\bv)$ in HTE, such as the most common standard normal distribution. However, to minimize the variance of HTE, we opt for the Rademacher distribution as follows: for each dimension of the vector $\bv \sim p(\bv)$, it is a discrete probability distribution that has a 50\% chance of getting +1 and a 50\% chance of getting -1. The proof for the minimal variance is given in \cite{skorski2021modern}.

\subsection{HTE for High-Dimensional Second-Order Parabolic PDEs}
In this subsection, we introduce the second-order parabolic equation under consideration. Subsequently, we elaborate on how the HTE technique is seamlessly incorporated into the PINNs framework, facilitating accelerated convergence and reduced memory consumption. Following this, we delve into the theoretical properties of the HTE loss functions. Lastly, we discuss the efficient implementation of HTE, providing a JAX \cite{jax2018github} pseudocode for HTE to aid readers in realizing its efficiency.

\subsubsection{Methodology}
We focus on a class of second-order parabolic equations, which include the Fokker-Planck equation in statistical mechanics, the Black-Scholes equation in mathematical finance, the Hamilton-Jacobi-Bellman equation in optimal control, Schr\"{o}dinger equation in quantum physics, etc., all ubiquitous in science and engineering:
\begin{equation}\label{eq:FP}
\begin{aligned}
\partial_t u(\bx, t)+ {\operatorname{Tr}\left(\sigma\sigma^\mathrm{T}(\bx, t)\left(\operatorname{Hess}_{\bx}u\right)(\bx, t)\right)} + f(\bx, t, u, \nabla_{\bx}u) = 0,\quad \bx \in \mathbb{R}^d, t \in [0,T],
\end{aligned}
\end{equation}
where $u(\bx, t)$ is the unknown exact solution we wish to solve, $\operatorname{Hess}_{\bx}u$ denotes the Hessian matrix of $u$, $\sigma(\bx, t) \in \mathbb{R}^{d \times d}$ is a known matrix-valued function, and $f(\bx, t, u, \nabla_{\bx} u)$ is also a known scalar function. PDEs with this form are of great interest in high-dimensions addressing the curse-of-dimensionality; see \cite{beck2021deep,han2018solving,hu2023tackling,hu2023bias,raissi2018forward}.
PINN's memory and speed bottleneck are traces of the second-order Hessian part, which is high-dimensional and high-order. Concretely, the computational cost of automatic differentiation in PINNs increases exponentially with the order \cite{bettencourt2019taylormode}. Regarding dimensionality, for second-order parabolic equations, the size of the Hessian matrix increases quadratically with the dimension, making high dimensionality and high-order the primary computational bottlenecks.

To overcome the speed and memory bottleneck due to the Hessian trace, we can use the Hutchinson Trace Estimation (HTE) to estimate the trace in the equation for efficient PINNs since computing Hessian-Vector Product (HVP) is much faster and more memory-efficient than the full Hessian.
The memory cost of HTE is significantly smaller than that of full PINNs and full Hessian since the output of HTE is a scalar, which reduces to $O(1)$ memory cost compared to the $O(d^2)$ in the full Hessian.

More specifically, the residual loss of the original PINN \cite{raissi2018forward} on a point $(\bx, t)$ is 
\begin{equation}\label{eq:PINN}
L_{\text{PINN}}(\theta) = \frac{1}{2}\left[\operatorname{Tr}\left(A_{\theta}(\bx, t)\right) + B_{\theta}(\bx, t)\right]^2, 
\end{equation}
where
$
A_\theta(\bx,t) := \sigma\sigma^\mathrm{T}(\bx, t)\left(\operatorname{Hess}_{\bx}u_\theta\right)(\bx, t) \in \mathbb{R}^{d \times d}$, $B_\theta(\bx, t) := \partial_t u_\theta(\bx, t)+f(\bx, t, u_\theta, \nabla_{\bx}u_\theta) \in \mathbb{R}$, and $u_\theta(\bx, t)$ is the PINN neural network model parameterized by $\theta$.

The loss of the HTE replaces the Hessian trace part with a stochastic trace estimator:
\begin{equation}\label{eq:HTE}
L_{\text{HTE}}(\theta;\{\bv_i\}_{i=1}^V) = \frac{1}{2}\left(\frac{1}{V}\sum_{i=1}^V\bv_i^\mathrm{T}A_\theta(\bx,t)\bv_i + B_{\theta}(\bx, t)\right)^2, 
\end{equation}
where we adopt an HTE with batch size $V$, i.e., we sample $\{\bv_i\}_{i=1}^V$ which are $i.i.d.$ samples from the distribution $p(\bv)$ such that $\mathbb{E}_{\bv \sim p(\bv)}[\bv \bv^\mathrm{T}] = I$, which we consider the Rademacher distribution to minimize its variance \cite{skorski2021modern}.

Although HTE is an unbiased trace estimator, the HTE loss in equation (\ref{eq:HTE}) is biased due to the nonlinear mean square error loss function, which breaks the linearity of mathematical expectation, i.e., the HTE loss in equation (\ref{eq:HTE}) is biased while it converges to the exact PINN loss almost surely (a.s.) as $V \rightarrow \infty$.
To correct the bias, similar to the techniques in Hu et al. \cite{hu2023bias}, we are required to sample two sets of samples:
\begin{align}\label{eq:HTE_unbiased}
L_{\text{HTE, unbiased}}(\theta;\{\bv_i,\hat{\bv}_i\}_{i=1}^V) &= \frac{1}{2}\left(\frac{1}{V}\sum_{i=1}^V\bv_i^\mathrm{T}A_\theta(\bx,t)\bv_i + B_{\theta}(\bx, t)\right)\left(\frac{1}{V}\sum_{i=1}^V\hat{\bv}_i^\mathrm{T}A_\theta(\bx,t)\hat{\bv}_i + B_{\theta}(\bx, t)\right),
\end{align}
where $\{\bv_i\}_{i=1}^V$ and $\{\hat{\bv}_i\}_{i=1}^V$ are 2$V$ $i.i.d.$ samples from the distribution $p(\bv)$. 
The properties of the two loss functions are summarized in the following theorem.

\begin{theorem}\label{thm:unbiased}
The loss $L_{\text{HTE}}(\theta)$ in equation (\ref{eq:HTE}) converges almost surely (a.s.) to the exact PINN loss $L_{\text{PINN}}(\theta)$ in equation (\ref{eq:PINN}), as $V \rightarrow \infty$, i.e.,
\begin{align}
\mathbb{P}\left(\lim_{V \rightarrow \infty}L_{\text{HTE}}(\theta;\{\bv_i\}_{i=1}^V) = L_{\text{PINN}}(\theta)\right) = 1.
\end{align}
The loss $L_{\text{HTE, unbiased}}(\theta)$ in equation (\ref{eq:HTE_unbiased}) is an unbiased estimator for the exact PINN loss $L_{\text{PINN}}(\theta)$ in equation (\ref{eq:PINN}), i.e.,
\begin{align}
\mathbb{E}_{\{\bv_i, \hat{\bv}_i\}_{i=1}^V}\left[L_{\text{HTE, unbiased}}(\theta;\{\bv_i,\hat{\bv}_i\}_{i=1}^V)\right] = L_{\text{PINN}}(\theta).
\end{align}
\end{theorem}
\begin{proof}
The proof is presented in \ref{appendix:unbiased}.
\end{proof}

\subsubsection{Discussion on the Bias and Unbiased HTE Versions}
In HTE, there exists a bias-variance tradeoff as in the Randomized Smoothing PINNs in \cite{hu2023bias}, where the biased version (equation (\ref{eq:HTE})) involves fewer samples, resulting in lower variance, while the unbiased version (equation (\ref{eq:HTE_unbiased})) employs more samples, leading to excessive variance. In practice, we utilize the biased version as it proves sufficient. This is attributed to the negligible bias of HTE in practical applications if $V$ is large enough, whose bias is much less obvious than the bias in Randomized Smoothing PINN considered in \cite{hu2023bias}. HTE's convergence results are already satisfactory. Indeed, HTE loss converges to the exact PINN loss, so its bias is negligible if $V$ is large enough.

Regarding the theoretical understanding of the biased version of HTE, this bias is proportional to the HTE-PINN residual variance. Concretely, its bias against the exact PINN loss $L_{\text{PINN}}(\theta)$ can be computed as follows:
\begin{equation}
\begin{aligned}
&\quad\mathbb{E}_{\{\bv_i\}_{i=1}^V}\left[L_{\text{HTE}}(\theta;\{\bv_i\}_{i=1}^V)\right] - L_{\text{PINN}}(\theta) \\
&= \frac{1}{2}\mathbb{E}_{\{\bv_i\}_{i=1}^V}\left[\left(\frac{1}{V}\sum_{i=1}^V\bv_i^\mathrm{T}A_\theta(\bx,t)\bv_i + B_{\theta}(\bx, t)\right)^2\right] - \frac{1}{2}\left[\operatorname{Tr}\left(A_{\theta}(\bx, t)\right) + B_{\theta}(\bx, t)\right]^2 \\
&= \frac{1}{2}\mathbb{E}_{\{\bv_i\}_{i=1}^V}\left[\left(\frac{1}{V}\sum_{i=1}^V\bv_i^\mathrm{T}A_\theta(\bx,t)\bv_i + B_{\theta}(\bx, t)\right)^2\right] - \frac{1}{2}\left(\mathbb{E}_{\{\bv_i\}_{i=1}^V}\left[\frac{1}{V}\sum_{i=1}^V\bv_i^\mathrm{T}A_\theta(\bx,t)\bv_i + B_{\theta}(\bx, t)\right]\right)^2 \\
&=\frac{1}{2}\mathbb{V}_{\{\bv_i\}_{i=1}^V}\left[\frac{1}{V}\sum_{i=1}^V\bv_i^\mathrm{T}A_\theta(\bx,t)\bv_i + B_{\theta}(\bx, t)\right].
\end{aligned}
\end{equation}
which is the variance of the HTE-PINN residual defined as $\frac{1}{V}\sum_{i=1}^V\bv_i^\mathrm{T}A_\theta(\bx,t)\bv_i + B_{\theta}(\bx, t)$, which replaces the full PINN residual by HTE.
Suppose learning rate decay is used (all our experiments adopt this common practice in machine learning). In that case, the scale of the loss for PINN optimization will become smaller, i.e., we are basically optimizing $\epsilon^2 \cdot L_{\text{HTE}}(\theta;\{\bv_i\}_{i=1}^V)$ with a small and decaying $\epsilon > 0$ in practice. Thus, the biased version of HTE's bias becomes $\epsilon$ times the HTE-PINN residual variance, i.e. $\epsilon \times \frac{1}{2}\mathbb{V}_{\{\bv_i\}_{i=1}^V}\left[\frac{1}{V}\sum_{i=1}^V\bv_i^\mathrm{T}A_\theta(\bx,t)\bv_i + B_{\theta}(\bx, t)\right]$, which decays gradually due to decaying $\epsilon$ to ensure a decreasing variance for better convergence.

In practice, the success of the biased version should be attributed to its relatively low bias. I.e., The biased version is the fastest due to sampling only once, and it can converge well thanks to the low bias in practice, as long as its batch size $V$ in HTE is not too small.

The corresponding experiments for the bias and unbiasedness are conducted in Section \ref{exp:bias_unbias}.

\subsubsection{HTE Implementation Detail}

We provide an implementation of HTE in a few lines of code using JAX \cite{jax2018github}. Although HTE is mathematically formulated as an HVP, the implementation details significantly impact its speed and memory usage. For instance, a straightforward computation of the entire Hessian followed by a vector multiplication will not yield any efficiency gains, where the curse of dimensionality would still plague as the computation complexity scales with the PDE problem's dimensionality. This is because we will need to compute all the Hessian trace elements using the standard basis and the number of HVP to compute scales quadratically with the dimension of the equation. The key reason that HTE+HVP works is that, instead of computing the exact hessian diagonal, we just compute a stochastic approximation of it, and the computation cost is amortized over the optimization process. With HTE, the number of HVPs to compute is constant regarding PDE dimensionality. Hence, the computation complexity scales much slower.
Below, we show a sample code for highly efficient HTE in JAX Taylor-mode automatic differentiation (AD), based on fully forward mode \cite{bettencourt2019taylormode}, which is super fast due to the absence of backward mode autodiff. The algorithm of HTE aligns well with efficient Taylor mode AD to avoid the slow naive implementation; thus, it is employed to implement HTE. At the time of writing, JAX is the only mature AD framework that provides Taylor mode AD, so we chose to implement our method in JAX.
\begin{figure}[htbp]
    \centering
    \includegraphics[scale=1]{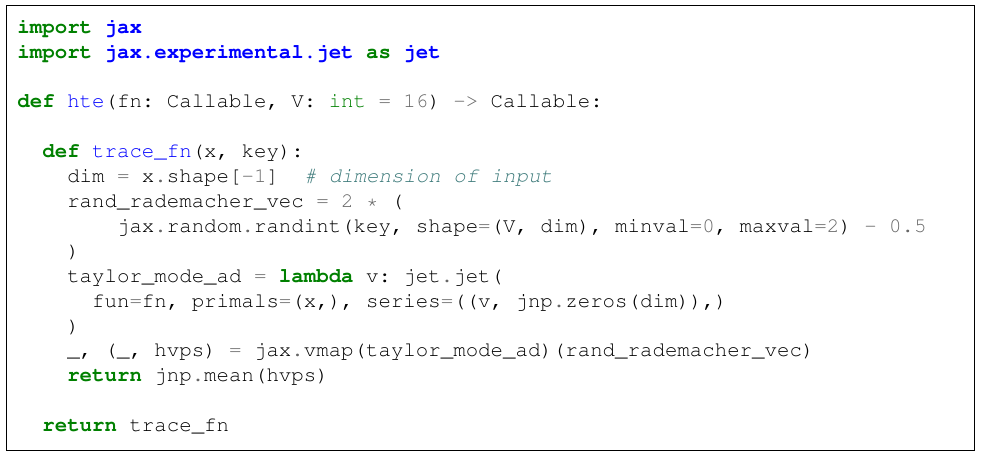}
\end{figure}
Specifically, given an input function fn, we return its HTE, $\frac{1}{V}\sum_{i=1}^V\bv_i (\text{Hess fn})\bv_i$, to estimate its Hessian trace.

Our approach and code implementation alleviate the cost of computing high-order derivatives by combining Taylor mode automatic differentiation (AD) and HTE, whose scaling is much better than the vanilla backward AD approaches widely adopted in the PINN literature.

In the AD framework, functions are implemented as the composition of primitives, which are simple functions with known analytical formulas for derivatives. We can write such function as $f = f_{n} \circ \dots \circ f_{1}$ where $f_{i}$ are the primitives, e.g., a neural network containing multiple linear layers and activations. Then by chain rule, the Jacobian of $f$ is simply the matrix product of the Jacobians of all the primitives $f_{i}$: $J_{f}=J_{f_{n}}\cdot \dots \cdot J_{f_{1}}$. Instead of computing the full Jacobians, in the AD framework, we usually compute a certain contraction of the derivative tensor, and it is computed by composing the analytical formula of the contraction considered for each primitive. The forward mode AD computes the Jacobi-Vector-Product (JVP): $J_{f}\bv=J_{f_{n}}(J_{f_{n-1}} (\dots (J_{f_{1}} \bv)))$, whereas the backward mode AD computes the Vector-Jacobian-Product (VJP): $\bv^{\mathrm{T}}J_{f}=(((\bv^\mathrm{T}J_{f_{n}})J_{f_{n-1}})\dots J_{f_{1}})$. The vector $\bv$ is called tangent in the case of JVP and cotangent in the case of VJP. From the perspective of differential geometry, the JVP of each primitive can be seen as the pushforward map from the input space to the next hidden space; hence, in AD, we also refer to the computation process of iteratively applying the JVP map of each primitive as ``pushing the tangent forward''. Similarly, VJP can be seen as the pullback of the input cotangent $\bv^{\mathrm{T}}$.

In the backward mode, one needs to first perform one forward pass to obtain the linearization points of all the primitive Jacobians, and this set of linearization points, also called the evaluation trace, needs to be stored in the memory. Furthermore, due to the forward pass, the number of sequential computations needed with the backward mode is $2n$ where $n$ is the number of primitives $\{f_i\}_{i=1}^n$, whereas for the forward mode, it is just $n$. This increases the computational complexity as well.

However, backward mode AD is the default choice for deep learning and PINN. This is because the Jacobian of a scalar loss function has the shape $1\times N$ where $N$ is the input dimension, which can be computed with one VJP, or $N$ JVP. Obviously, it would be much more efficient to do one VJP instead of $N$ JVP, even though for one contraction, VJP is more expensive than JVP in both speed and memory.

Therefore, the most widely adopted way in the PINN community to compute higher-order derivatives in the residual loss is to use repeated application of backward mode AD. But with each repeated application of backward mode AD, one doubles the amount of sequential computation and also the number of tensors to store in the evaluation trace for the backward pass. Furthermore, the new element to store in the evaluation trace has a dimension that grows exponentially to the derivative order. In this case, computing high-order derivatives with standard backward AD is indeed inefficient.
Repeated application of forward mode AD is also inefficient, as the size of the input tangent $\bv$ grows exponentially to the derivative order.

In this paper, we use Taylor mode AD, a generalization of forward mode AD that provides composition rules for a specific set of higher-order derivative tensor contractions related to the high-order Taylor expansion. In forward mode AD, only the JVP contraction is pushed forward, but in Taylor mode AD, we are pushing forward multiple contractions of derivatives tensor up to the highest orders. This avoids repeated application of AD when computing high-order derivatives.
And just like forward mode AD, the huge derivative tensor was never materialized in Taylor mode AD, and only the contraction (a scalar) was computed. The memory complexity is, therefore, much lower than backward mode AD since no evaluation trace is required, and only the contracted quantities (a scalar) are stored in memory. Furthermore, since the number of random vectors $V$ in the HTE is fixed and unrelated to the dimension, the curse of dimensionality is also alleviated.

\subsection{Comparison with SDGD}
Recently, the related work of SDGD \cite{hu2023tackling} has demonstrated robust capabilities in tackling high-dimensional PDEs and the curse-of-dimensionality. SDGD primarily relies on sampling dimensions to mitigate the computational and memory burdens associated with increasing dimensions. In contrast, HTE employs HVP to avoid the challenges posed by computing the full Hessian, whose cost grows quadratically with dimensionality. In this subsection, we will compare their variances and discuss their similarity.

\subsubsection{SDGD as a Special Case of HTE}
We show that SDGD can be regarded as a special case of HTE. Consider the discrete distribution $p(\bv)$ such that $\bv = \sqrt{d}\boldsymbol{e}_i$ for $i=1,2,\cdots,d$ with probability $1 / d$, where $\boldsymbol{e}_i$ denotes the standard basis vector with a one in the $i$th coordinate and zero elsewhere. Then, $\mathbb{E}_{\bv \sim p(\bv)}[\bv\bv^\mathrm{T}] = I$, and the corresponding estimator is $\operatorname{Tr}(A) \approx \frac{d}{|I|}\sum_{i\in I}A_{ii}$ where $I \subset \{1,2,\cdots,d\}$ is a multiset and $|I|$ denotes its cardinality, which is exactly sampling dimensionality as SDGD. Despite the theoretical similarity, SDGD differs in the following aspects:
\begin{itemize}
\item To minimize gradient variance, SDGD employs sampling without replacement, meaning each dimension is selected at most once. On the other hand, if we consider SDGD as a special case of HTE, it uses sampling with resampling, i.e., a dimension can be sampled several times; thus, the sampled dimension indices, denoted as $I$, form a multiset instead of a set in the original SDGD formulation.
\item While SDGD can be considered a special case of HTE, implementing SDGD using HTE is inefficient. This is because the probability distribution in this context takes values in a sparse vector, where most elements are zero. Optimization for sparse vectors can be applied in this scenario. However, HTE is generally implemented using non-sparse vectors, and this advantage cannot be fully exploited in this specific case.
\end{itemize}
Overall, HTE and SDGD are similar in some aspects but differ in other viewpoints.

\subsubsection{Variance Comparison between HTE and SDGD}
Both SDGD and HTE can speed up and scale up PINNs in high-dimensional and high-order PDEs. However, their source of variance differs due to various sampling and estimations.
Specifically, we focus on the second-order parabolic PDE given in equation (\ref{eq:FP}), and SDGD and HTE estimate the trace of 2nd-order derivatives
$
\operatorname{Tr}(A) = \sum_{i=1}^d A_{ii},
$
where $A = \sigma\sigma^\mathrm{T}(\bx, t)\left(\operatorname{Hess}_{\bx}u\right)(\bx, t)$. For simplicity, we drop $A$'s dependency on $\theta, \bx, t$. In the following theorems and their corresponding remarks, we prove the variance of SDGD and HTE and explain their various sources.
\begin{theorem}\label{thm:sdgd_var}
Given the SDGD estimator of the trace 
$
\operatorname{Tr}(A) \approx \frac{d}{B}\sum_{i\in I}A_{ii}
$
given an index set $I \subset \{1,2,\cdots,d\}$ whose cardinality $|I| = B$, where $B$ is the SDGD's batch size for the dimension, its variance is
\begin{equation}
\begin{aligned}
\mathbb{V}\left[\frac{d}{B}\sum_{i\in I}A_{ii}\right] = \frac{1}{\binom{d}{B}}\left(\sum_{I: |I| = B}\frac{d}{B}\sum_{i\in I}A_{ii} - \sum_{i=1}^dA_{ii}\right)^2,
\end{aligned}
\end{equation}
where $\binom{d}{B} = \frac{d(d-1)\cdots(d-B+1)}{B(B-1)\cdots1}$.
\end{theorem}
\begin{proof}
The proof is presented in \ref{appendix:sdgd_var}.
\end{proof}
\begin{remark}
SDGD's variance comes from the variance between diagonal elements across different dimensions.
\end{remark}
\begin{theorem}\label{thm:hte_var}
Given the HTE estimator of the trace 
$
\operatorname{Tr}(A) \approx\frac{1}{V}\sum_{k=1}^V \bv_k^\mathrm{T} A\bv_k,
$ 
where $V$ is the HTE batch size and each dimension of $\bv_k \in \mathbb{R}^d$ is an $i.i.d.$ sample from the Rademacher distribution, then its variance is
$
\frac{1}{V} \sum_{i \neq j}A_{ij}^2.
$
\end{theorem}
\begin{proof}
The proof is presented in \ref{appendix:hte_var}.
\end{proof}
\begin{remark}
HTE's variance only comes from the off-diagonal elements.
\end{remark}
Based on this observation, the following statements can be inferred:
\begin{itemize}
\item If the diagonal elements are similar, i.e., the PDE exact solution is symmetric, then SDGD has low variance.
\item If the off-diagonal elements of the Hessian are zero, i.e., different dimensions do not interact, then HTE is exact.
\item If the scales of off-diagonal elements are much larger than the diagonal ones, then HTE suffers from huge variance. 
\end{itemize}
Hence, we propose the following examples to showcase when HTE or SDGD can outperform each other and when they perform similarly:
\begin{itemize}
\item \textbf{SDGD fails but HTE is accurate}. Consider the 2D exact solution $f(x, y) = -kx^2 + ky^2$ and its Laplacian $\Delta f(x, y) = 0$, where $k \in \mathbb{R}^+$ is large enough. For SDGD, we choose the batch size of dimension as 1. Thus, the estimator of SDGD will either be $\frac{\partial^2 f(x, y)}{\partial x^2} = -2k$ or $\frac{\partial^2 f(x, y)}{\partial y^2} = 2k$, whose variance is $4k^2$. For HTE, since the off-diagonal elements in the Hessian matrix of $f$ are zero, HTE is exact, i.e., has zero variance. In brief, for this case, SDGD suffers from the variance $4k^2$ when $k$ is large, while HTE's variance is zero. 
\item \textbf{HTE fails but SDGD is accurate}. Consider the 2D exact solution $f(x, y) = kxy$ and its Laplacian $\Delta f(x, y) = 0$, where $k \in \mathbb{R}^+$ is large enough. Since the diagonal of its Hessian matrix is all zero, SDGD is exact, i.e., SDGD has zero variance since $\frac{\partial^2 f(x, y)}{\partial x^2} = \frac{\partial^2 f(x, y)}{\partial y^2} = 0$ is exact. For HTE, suppose that its batch size is $V=1$. Then HTE's estimator is $2v_1v_2\frac{\partial^2 f(x, y)}{\partial x\partial y} = 2v_1v_2k$ where $v_1$ and $v_2$ follows the Rademacher distribution. Thus, HTE's variance is $4k^2$. In brief, for this case, HTE suffers from the variance $4k^2$ when $k$ is large, while SDGD's variance is zero and SDGD is exact. 
\item \textbf{HTE and SDGD have the same nonzero variance}. Consider the 2D exact solution $f(x, y) = k(-x^2 + y^2 + xy)$ and its Laplacian $\Delta f(x, y) = 0$. For SDGD, we choose the batch size of dimension as 1. Thus, the estimator of SDGD will either be $\frac{\partial^2 f(x, y)}{\partial x^2} = -2k$ or $\frac{\partial^2 f(x, y)}{\partial y^2} = 2k$, whose variance is $4k^2$. For HTE, suppose that its batch size is $V=1$. Then HTE's estimator is $2v_1v_2\frac{\partial^2 f(x, y)}{\partial x\partial y} = 2v_1v_2k$, where $v_1$ and $v_2$ follows the Rademacher distribution. Thus, HTE's variance is $4k^2$. In brief, HTE and SDGD have the same nonzero variance $4k^2$.
\end{itemize}
However, we emphasize that the specific use case can be much more complex. Since we are applying SDGD or HTE to the neural network at each iteration, the characteristics of the function represented by the network determine which method is superior. This, in turn, is challenging to observe directly. Therefore, we provide these insights to readers as a basis for consideration. In practical applications, SDGD and HTE algorithms can be judiciously chosen based on some {\em a priori} knowledge of the PDE problem and the true solution. One can explore the simultaneous application of two methods and observe the rate of training loss reduction for each. Given the strong correlation between the error in approximating the true solution by PINN and the training loss \cite{hu2021extended}, algorithms with lower variance are more likely to converge faster.

\subsection{Applications to Biharmonic Equations}
Given that HTE demonstrates superior performance for high-order and high-dimensional PDEs, we illustrate its application to the biharmonic equation in this subsection. The biharmonic equation is a fourth-order PDE that can be defined in any dimension. We will demonstrate how to extend HTE, involving the generalization of HVP to a tensor-vector product, and provide a JAX pseudo-code for efficient implementation.

Concretely, the fourth-order and $d$-dimensional Biharmonic operator is given by
\begin{align}
\Delta^2 u(\bx) = \sum_{i=1}^d\sum_{j=1}^d \frac{\partial^4}{\partial\bx_i^2\partial \bx_j^2}u(\bx).
\end{align}
The biharmonic equation is crucial in understanding and simulating complex behaviors. It is a cornerstone in analyzing systems governed by higher-order differential equations, including modeling elastic membranes, thin plates, and stream functions in fluid dynamics.
We will attempt to generalize HTE, starting with the extension of HVP to a tensor-vector product (TVP), to solve high-dimensional and high-order Biharmonic equations efficiently.
\begin{theorem}\label{thm:biharmonic}
The biharmonic operator can be unbiasedly estimated using the following TVPs:
\begin{equation}
\begin{aligned}
\Delta^2 u(\bx) &= \frac{1}{3}\mathbb{E}_{\bv\sim \mathcal{N}(0, I)}\left[\frac{\partial^4}{\partial\bx^4}u(\bx)[\bv, \bv, \bv, \bv] \right],
\end{aligned}
\end{equation}
where $\mathcal{N}(0, I)$ is the $d$-dimensional unit Gaussian, and
\begin{equation}
\frac{\partial^4}{\partial\bx^4}u(\bx) \in \mathbb{R}^{d \times d \times d \times d}, \quad \text{where }\left[\frac{\partial^4}{\partial\bx^4}u(\bx)\right]_{ijkl} = \frac{\partial^4}{\partial\bx_i\partial\bx_j\partial\bx_k\partial\bx_l}u(\bx),
\end{equation}
and the TVP is defined as 
\begin{equation}
\frac{\partial^4}{\partial\bx^4}u(\bx)[\bv, \bv, \bv, \bv] = \sum_{i,j,k,l=1}^d\frac{\partial^4}{\partial\bx_i\partial\bx_j\partial\bx_k\partial\bx_l}u(\bx)\bv_i\bv_j\bv_k\bv_l,
\end{equation}
\end{theorem}
\begin{proof}
The proof is presented in \ref{appendix:biharmonic}.
\end{proof}
Therefore, we can also obtain an unbiased estimate of the biharmonic operator through sampling. 
To help the reader understand the implementation of the tensor vector product, we provide a pseudocode implementation based on JAX \cite{jax2018github} Taylor Mode automatic differentiation \cite{bettencourt2019taylormode}. We want to emphasize again that the implementation details significantly impact the speed and memory
usage. A straightforward computation of the entire fourth-order derivative tensor followed by the vector multiplication will not yield any efficiency gains. Rather, we should leverage Taylor-mode automatic differentiation to allow JAX to implicitly compute the Tensor-Vector Product (TVP), thereby avoiding the massive memory consumption associated with explicit calculations. Furthermore, internal optimizations within JAX can be exploited to enhance speed by directly invoking JAX's TVP functionality and extracting only the output TVP, as opposed to the entire fourth-order derivative tensor.
\begin{figure}[htbp]
    \centering
    \includegraphics[scale=1]{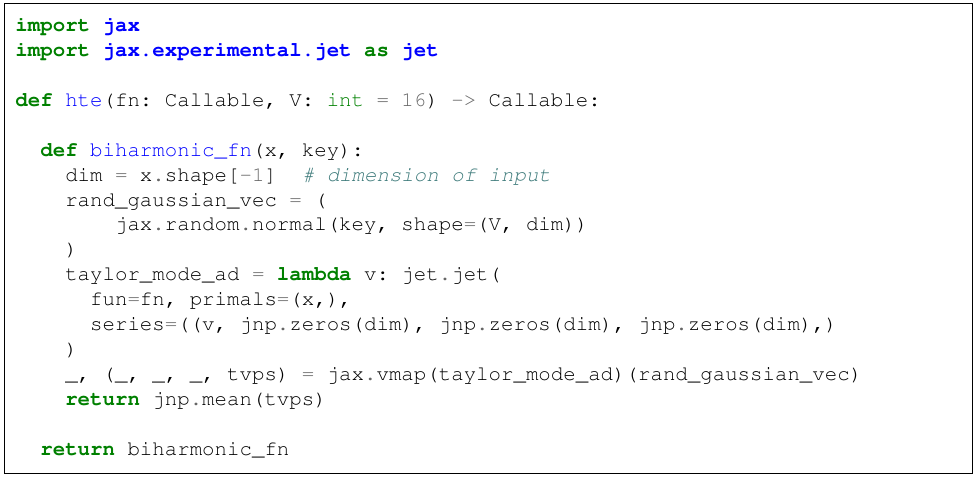}
\end{figure}
Similar to our previous treatment of the HTE for Hessian trace estimation, we analyze the HTE's variance when estimating the biharmonic operator. Here, we employ a Gaussian distribution instead of a Rademacher distribution, and thus both diagonal and off-diagonal elements contribute to a certain extent of variance. Consequently, a larger HTE batch size, denoted as $V$, is required when estimating the biharmonic operator.

\subsection{Applicability of HTE}
Next, we move to discuss the applicability of HTE.
HTE is more beneficial in higher dimensions and higher orders; otherwise, HTE's acceleration effect will be less obvious, although it is still faster. 
More generally, from a mathematical perspective, for a $d$-dimensional $n$th-order PDE, the pure autograd computation results in a tensor of dimensions $d^n$. For example, the Jacobian is of dimension $d$, and the Hessian is of dimension $d^2$. In contrast, the output of HTE are $V$ scalars when using $V$ as the batch size; thus, the output dimension is $V$. Consequently, as the dimension and order increase, the advantages of HTE become more pronounced.

\subsubsection{HTE's Application to Other Neural PDE Solvers}
Furthermore, our HTE approach covers various indispensable PDE differential operators. Hence, our HTE can be applied to other deep learning-based PDE solvers, such as the deep Ritz \cite{Weinan2017TheDR} and weak adversarial network (WAN) approaches \cite{zang2020weak} in addition to PINN \cite{raissi2019physics}. 

Specifically, the deep Ritz method \cite{Weinan2017TheDR} considers variational problems with the differential operator $\Vert \nabla_{\bx} u \Vert^2$, which can be efficiently estimated via
$
\Vert \nabla_{\bx} u(\bx) \Vert^2 = \mathbb{E}_{\bv \sim p(\bv)} |\bv^\mathrm{T}\nabla_{\bx}u(\bx)|^2,
$
where $\mathbb{E}_{\bv \sim p(\bv)}[\bv\bv^\mathrm{T}] = I$ as in our HTE method. The randomization of our HTE simplifies the original $O(d)$ ($d$ is the problem dimension) neural network gradients with respect to $\bx \in \mathbb{R}^d$ into $O(1)$ Jacobian vector products (JVPs), thereby conserving memory consumption and accelerating computation. Here, the JVP employed is essentially a specific instance of the proposed Tensor-Vector Product (TVP) in our HTE method, where the tensor corresponds to the Jacobian.

Besides, Weak Adversarial Networks (WANs) \cite{zang2020weak} consider PDE weak formulations, which consider the second-order elliptic PDEs and second-order parabolic PDEs. These PDEs considered by WAN are a subset of the cases in our HTE paper with the following form:
$
\partial_t u(\bx, t)+ {\operatorname{Tr}\left(\sigma\sigma^\mathrm{T}(\bx, t)\left(\operatorname{Hess}_{\bx}u\right)(\bx, t)\right)} + f(\bx, t, u, \nabla_{\bx}u) = 0, \bx \in \mathbb{R}^d, t \in [0,T],
$
Due to the same reason, our HTE can also be adopted to accelerate the convergence of WAN.
In summary, our HTE serves as an acceleration technique for general high-dimensional and high-order neural PDE solvers, making it widely applicable in PINN, deep Ritz method, and WAN.

\subsubsection{Discussion on HTE's Applicable PDE Types}
Furthermore, regarding the types of PDEs where HTE can be applied, our primary focus is on high-order and high-dimensional PDEs, where high-dimensional PDEs mainly consist of second-order elliptic and parabolic equations, which have wide-ranging applications \cite{beck2021deep,han2018solving,hu2023tackling,hu2023bias,raissi2018forward}, which covers the Fokker-Planck equation from statistical mechanics to describe the evolution of stochastic differential equations (SDEs), the Black-Scholes equation in mathematical finance, the Hamilton-Jacobi-Bellman equation in optimal control, and Schr\"{o}dinger equation in quantum physics. Although it is possible to design irregular equations that render HTE unusable, such equations lack practical applications, and even the uniqueness of their solutions is not theoretically guaranteed. 
To the best of our knowledge, the high-dimensional PDEs HTE can solve can cover the majority of real-world PDEs with multiple applications.
Additionally, for high-order equations, we consider high-order and high-dimensional biharmonic equations modeling elastic membranes, thin plates, and stream functions in fluid dynamics.

\subsubsection{Discussion on Low-Dimensional and High-Order PDEs}
Our HTE directly applies to high-order KdV and Kuramoto-Sivashinsky equations thanks to the efficient Taylor-mode automatic differentiation in HTE. 
More specifically, our HTE and the provided pseudocode can be directly applied to these high-order equations, with the distribution of the vector 
$\bv$ being Rademacher distribution. Just to note that these equations are one-dimensional in $\bx$, which is low-dimensional. In the case of one-dimensional but higher-order scenarios, the primary acceleration of HTE stems from Taylor-mode automatic differentiation, which, as previously mentioned, is significantly faster than the backward mode and forward mode conventionally employed in deep learning and PINN. The acceleration provided by our HTE becomes more pronounced as the dimensionality and order of the PDE problem increase. It often enables handling challenging cases where standard PINN would exceed GPU memory capacity or converge extremely slowly.

\section{Computational Experiments}
\subsection{Nonlinear Second-Order PDEs with Anisotropic and Inseparable Solutions}
This subsection considers nonlinear second-order Sine-Gordon PDEs with anisotropic and inseparable solutions.
The following complicated exact solution is considered in related work  \cite{hu2023tackling,hu2023bias} to test the convergence of high-dimensional PDE solvers:
\begin{equation}\label{eq:sol_2body}
u_{\text{exact}}(\bx) = \left(1 - \Vert \bx \Vert_2^2\right)\left(\sum_{i=1}^{d-1}  c_i \sin(\bx_i +\cos(\bx_{i+1})+\bx_{i+1}\cos(\bx_i))\right),
\end{equation}
where $c_i \sim \mathcal{N}(0, 1)$.
The term $1 - \Vert \bx \Vert_2^2$ is multiplied for a zero boundary condition on the unit ball to prevent the boundary loss from leaking the solution's information. The solution's $i$th and $(i+1)$th dimension correlate, also known as the two-body interaction.

Based on the idea, we further introduce the following complicated exact solutions to test various algorithms:
\begin{equation}\label{eq:sol_3body}
u_{\text{exact}}(\bx) = \left(1 - \Vert \bx \Vert_2^2\right)\left(\sum_{i=1}^{d-2}  c_{i} \exp(\bx_i\bx_{i+1}\bx_{i+2})\right),
\end{equation}
where the three-body interaction between each pair $(\bx_i, \bx_{i+1}, \bx_{i+2})$ is considered.

As for the nonlinear PDE, the Sine-Gordon equation with zero boundary condition is considered
\begin{align}
&\Delta u(\bx) + \sin\left(u(\bx) \right) = g(\bx), \quad \bx \in \mathbb{B}^d= \{\bx \ | \ \Vert \bx \Vert_2 < 1\},\\
&u(\bx) = 0, \quad \bx \in \mathbb{S}^{d-1} = \{\bx \ |\ \Vert \bx \Vert_2 = 1\},
\end{align}
where $g(\bx) = \Delta u_{\text{exact}}(\bx) + \sin\left(u_{\text{exact}}(\bx) \right)$, since its nonlinearity order is infinity, i.e., its nonlinearity is strong.

Here are the implementation details. The model is a 4-layer fully connected network with 128 hidden units activated by Tanh, which is trained via Adam \cite{kingma2014adam} for 10K epochs in the two-body case and 20K epochs in the three-body case, respectively, with an initial learning rate 1e-3, which linearly decays to zero at the end of the optimization. We select 100 random residual points at each Adam epoch and 20K fixed testing points uniformly from the unit ball. We adopt the following model structure to satisfy the zero boundary condition with hard constraint and to avoid the boundary loss \cite{lu2021physics}
$
 (1 - \Vert\bx\Vert_2^2) u_\theta(\bx),
$
where $u_\theta(\bx)$ is the neural neural network model. We repeat our experiment five times with five independent random seeds and report the mean and standard deviation of the errors.
For HTE, we choose $V=16$. For SDGD, we chose the batch size of dimension 16.

\begin{table}[htbp]
\footnotesize
\centering
\begin{tabular}{|c|c|c|c|c|c|c|}
\hline
Method & Metric & 100 D & 1,000 D & 5,000D & 10,000 D & 100,000 D \\ \hline\hline
\multirow{4}{*}{PINNs \cite{raissi2019physics}} & Speed & 1367.61it/s & 141.45it/s & 8.43it/s & N.A. & N.A. \\ \cline{2-7} 
 & Memory & 1123MB & 2915MB & 14283MB & $>$80GB & $>$80GB \\ \cline{2-7} 
 & Error\_1 & 6.24E-3$\pm$2.83E-3 & 1.20E-3$\pm$6.47E-4 & 2.64E-3$\pm$1.33E-3 & N.A. & N.A. \\ \cline{2-7} 
 & Error\_2 & 7.43E-3$\pm$4.56E-4 & 1.00E-3$\pm$1.45E-5 & 2.02E-4$\pm$2.45E-6 & N.A. & N.A. \\ \hline
\hline
 \multirow{3}{*}{SDGD \cite{hu2023tackling}} & Speed & 1853.67it/s &  1444.00it/s &  1243.91it/s & 1027.38it/s & 356.62it/s  \\ \cline{2-7} 
 & Memory &883MB & 885MB & 909MB & 929MB & 1173MB  \\ \cline{2-7} 
 & Error\_1 & 6.28E-3$\pm$2.55E-3 & 1.55E-03$\pm$6.72E-4 & 2.63E-3$\pm$1.21E-3 & 1.93E-3$\pm$7.81E-4 & 2.30E-3$\pm$1.33E-3 \\ \cline{2-7} 
 & Error\_2& 7.60E-3$\pm$6.00E-4 & 1.15E-3$\pm$2.80E-4 & 2.05E-4$\pm$3.98E-6 & 1.04E-4$\pm$4.30E-6 & 1.56E-5$\pm$9.57E-6 \\ \hline\hline
\multirow{3}{*}{HTE (Ours)} & Speed & 1792.03it/s  & 1566.19it/s & 1275.81it/s & 1066.10it/s  &  345.10it/s \\ \cline{2-7} 
 & Memory & 869MB & 890MB & 913MB & 927MB & 1089MB \\ \cline{2-7} 
 & Error\_1 & 6.30E-3$\pm$2.88E-3 & 1.25E-3$\pm$3.40E-4 & 2.61E-3$\pm$1.31E-3 & 1.84E-3$\pm$9.56E-4 & 2.38E-3$\pm$1.72E-3 \\ \cline{2-7} 
 & Error\_2 & 7.58E-3$\pm$5.39E-4 & 9.94E-4$\pm$3.23E-6 & 2.01E-4$\pm$2.27E-6 & 1.02E-4$\pm$3.88E-6 & 1.59E-5$\pm$9.29E-6 \\ \hline
\end{tabular}
\caption{Computational results for Sine-Gordon equations with the two-body (Error\_1) and three-body (Error\_2) exact solutions. Under the same computational constraints, SDGD \cite{hu2023tackling} and our HTE perform similarly, and they are all competitive against vanilla PINNs \cite{raissi2019physics}.}
\label{tab:2}
\end{table}

The computational results for the two-body and three-body cases are presented in Table \ref{tab:2}, where we compare vanilla PINN \cite{raissi2019physics}, SDGD \cite{hu2023tackling}, and HTE (Ours), in terms of speed (iteration per second), memory cost (MB), and relative $L_2$ error for the two-body case (Error\_1) and that for the three body case (Error\_2). Note that when solving for the two PDEs with different solutions, the speed and the memory costs are similar, and only the error differs, so we only report speed and memory once. 
We use MB as the memory unit because it is more precise. Typically, when checking the GPU memory using the Linux command ``nvidia-smi", the output is in MB, and converting it to GB may result in multiple decimal places, which is not as clear and accurate as a whole number in MB.

Regarding the regular PINNs, we observe a significant slowdown in speed and a rapid increase in memory consumption as the dimensionality rises. This is because the regular PINNs require the computation of the full Hessian, and the computational cost quadratically increases with dimension. As a result, regular PINNs exceed the memory limit of an A100 GPU in the case of 10K dimensions. For HTE and SDGD, we observe similar convergence speeds, results and memory consumption. In scenarios with 100D, 1K D, and 5K D, the final errors are comparable to regular PINNs, demonstrating that HTE can overcome the curse-of-dimensionality. It accelerates regular PINNs without compromising its effectiveness. Moreover, the PDEs we used cannot be simplified into lower-dimensional subproblems as they are inherently high-dimensional.  We also conducted tests on the PDE exact solutions of two-body and three-body interactions to demonstrate the efficacy of HTE across various high-dimensional problems.

\subsubsection{Additional Study 1: Effect of $V$ on HTE's Convergence}
\begin{table}[htbp]
\small
\centering
\begin{tabular}{|c|c|c|c|c|c|}
\hline
$V$ & 1 & 5 & 10 & 15 & 16 \\ \hline
Speed & 435.63it/s & 402.10it/s & 382.91it/s & 350.18it/s & 345.10it/s \\ \hline
Memory & 1073MB & 1075MB & 1080MB & 1085MB & 1089MB \\ \hline
Error\_1 & 3.02E-03$\pm$2.16E-03 & 2.70E-03$\pm$1.35E-03 & 2.55E-03$\pm$1.08E-03 & 2.41E-03$\pm$8.10E-04 & \textbf{2.38E-3$\pm$1.72E-3} \\ \hline
Error\_2 & 2.23E-05$\pm$5.00E-05 & 1.96E-05$\pm$2.91E-05 & 1.75E-05$\pm$3.28E-05 & 1.62E-05$\pm$1.25E-05 & \textbf{1.59E-5$\pm$9.29E-6} \\ \hline
\end{tabular}
\caption{Effects of the number of random vectors $V$ on the convergence for 100,000-dimensional Sine-Gordon equations with the two-body (Error\_1) and three-body (Error\_2) exact solutions.}
\label{tab:sine-gordon-v}
\end{table}
We conduct additional experiments for the case of the highest 100,000-dimensional Sine Gordon equation by following Section 4.1 in the main paper, except with different $V$. We aim to understand the effect of HTE batch size $V$ on its convergence, speed, and memory usage. Table \ref{tab:sine-gordon-v} shows that $V=1$ is sufficient for a good convergence and that the error is generally reduced by increasing the HTE batch size $V$. Increasing the batch size results in slightly larger memory consumption and slower speed, as larger batch sizes entail greater computational load. These results demonstrate the robustness of HTE under minimal batch size $V$ and that we can always improve HTE by increasing $V$ at the cost of more memory usage and slower speed. This result also demonstrates the convergence rate of HTE towards that of the full PINN as $V$ increases, noting that we observe that HTE with $V=16$ is essentially indistinguishable from the PINN.

The philosophy is that optimizing and learning part of the dimensions is enough to achieve good performances in high-dimensional PDEs. In the PINN community, people are sampling residual points to accelerate convergence and reduce computational cost. HTE is, in essence, sampling dimensions while keeping a minibatch stochastic gradient.

The design philosophy of HTE is also akin to stochastic gradient descent (SGD) in machine learning. Given a large number of data points, we do not compute and optimize the loss for each data point at every iteration. Instead, we select a minibatch to construct an unbiased gradient, achieving good results. Today, SGD is the standard method in machine learning. Similarly, in high-dimensional PDEs, we do not need to optimize each dimension at every iteration. The randomized gradient obtained through HTE is sufficient.

\subsubsection{Additional Study 2: Effect of Biased and Unbiased Versions of HTE}\label{exp:bias_unbias}

\begin{table}[htbp]
  \footnotesize
\centering
\begin{tabular}{|c|c|c|c|c|c|c|}
\hline
Method & Metric & 100 D & 1,000 D & 5,000D & 10,000 D & 100,000 D \\ \hline\hline
\multirow{4}{*}{Biased HTE} & Speed & 1792.03it/s & 1566.19it/s & 1275.81it/s & 1066.10it/s & 345.10it/s \\ \cline{2-7} 
 & Memory & 869MB & 890MB & 913MB & 927MB & 1089MB \\ \cline{2-7} 
 & Error\_1 & 6.30E-3$\pm$2.88E-3 & 1.25E-3$\pm$3.40E-4 & 2.61E-3$\pm$1.31E-3 & 1.84E-3$\pm$9.56E-4 & 2.38E-3$\pm$1.72E-3 \\ \cline{2-7} 
 & Error\_2 & 7.58E-3$\pm$5.39E-4 & 9.94E-4$\pm$3.23E-6 & 2.01E-4$\pm$2.27E-6 & 1.02E-4$\pm$3.88E-6 & 1.59E-5$\pm$9.29E-6 \\ \hline\hline
\multirow{4}{*}{Unbiased HTE} & Speed & 1619.81it/s & 1437.20it/s & 1139.63it/s & 958.64it/s & 303.73it/s \\ \cline{2-7} 
 & Memory & 881MB & 901MB & 920MB & 932MB & 1094MB \\ \cline{2-7} 
 & Error\_1 & \textbf{6.21E-3$\pm$3.01E-3} & \textbf{1.23E-3$\pm$4.02E-4} & \textbf{2.48E-3$\pm$1.50E-3} & \textbf{1.79E-3$\pm$8.24E-4} & \textbf{2.33E-3$\pm$1.67E-3} \\ \cline{2-7} 
 & Error\_2 & \textbf{7.43E-3$\pm$5.02E-4} & \textbf{8.72E-4$\pm$1.09E-5} & \textbf{1.85E-4$\pm$2.54E-6} & \textbf{9.84E-5$\pm$3.92E-6} & \textbf{1.45E-5$\pm$1.12E-5} \\ \hline
\end{tabular}
\caption{Effects of the biased and unbiased versions of HTE with $V=16$ on the convergence for Sine-Gordon equations with the two-body (Error\_1) and three-body (Error\_2) exact solutions.}
\label{tab:sine-gordon-v2}
\end{table}

Table \ref{tab:sine-gordon-v2} presents additional experimental results for comparing the biased and unbiased versions of HTE. We conducted experiments following the first Sine-Gordon equations presented in the main text, utilizing the same hyperparameters, except we compared the biased and unbiased versions. The results indicate that the unbiased version is approximately 10\% slower than the biased version, with slightly higher memory consumption and a minor improvement in error. The increased cost and slower speed of the unbiased version stem from the necessity to compute HTE twice using two groups of independent $\{\bv_i\}_{i=1}^V$ and $\{\hat{\bv_i}\}_{i=1}^V$, whereas the biased version only requires one computation for $\{\bv_i\}_{i=1}^V$. However, while the unbiased version's unbiased gradients lead to a slight improvement in performance, in most cases, the biased version's performance is already sufficiently good.

\subsection{High-Dimensional Gradient-Enhanced PINN}
The Gradient-Enhanced Physics-Informed Neural Network (gPINN) \cite{yu2022gradient} is a technique designed to enhance the performance of PINNs. It achieves this by incorporating additional regularization through the computation of an extra first-order derivative for the PINN's residual loss. However, due to the extra derivatives, gPINN becomes computationally expensive, particularly in higher dimensions, where the required additional derivatives are proportional to the dimensionality of the PDE. In this subsection, we illustrate how to leverage the proposed HTE method to expedite the training of gPINN. To illustrate, we continue to consider the Sine-Gordon equation with a two-body interactive solution discussed in the preceding section in equation (\ref{eq:sol_2body}).

Given the Sine-Gordon equation
\begin{align}
&\Delta u(\bx) + \sin\left(u(\bx) \right) = g(\bx), \quad \bx \in \mathbb{B}^d= \{\bx \ | \ \Vert \bx \Vert_2 < 1\},\\
&u(\bx) = 0, \quad \bx \in \mathbb{S}^{d-1} = \{\bx \ |\ \Vert \bx \Vert_2 = 1\},
\end{align}
where $g(\bx) = \Delta u_{\text{exact}}(\bx) + \sin\left(u_{\text{exact}}(\bx) \right)$, the PINN loss is given by
\begin{align}
L_{\text{PINN}}(\theta) = \frac{1}{2}\left(\Delta u_\theta(\bx) + \sin\left(u_\theta(\bx)\right)-g(\bx)\right)^2 := \frac{1}{2}r_\theta(\bx)^2,
\end{align}
where $r_\theta(\bx) = \Delta u_\theta(\bx) + \sin\left(u_\theta(\bx)\right)-g(\bx) $ is the residual prediction.
The gPINN loss is given by
\begin{align}
L_{\text{gPINN}}(\theta) = \frac{1}{2}r_\theta(\bx)^2 + \frac{1}{2} \lambda_{\text{gPINN}} \left\|\nabla_{\bx}r_\theta(\bx)\right\|^2,
\end{align}
i.e., gPINN optimizes both the residual part and its derivative to be zero, where $\lambda_{\text{gPINN}} \in \mathbb{R}^+$ is the weight for the gPINN regularization.
To accelerate gPINN, we use the HTE-based gPINN loss
\begin{align}
L_{\text{HTE, gPINN}}(\theta) = \frac{1}{2}\hat{r}_\theta(\bx)^2 + \frac{1}{2} \lambda_{\text{gPINN}} \left\|\nabla_{\bx}\hat{r}_\theta(\bx)\right\|^2,
\end{align}
where $\hat{r}_\theta(\bx) =\frac{1}{V}\sum_{i=1}^V\bv_i^\mathrm{T}\left(\operatorname{Hess}u_\theta(\bx)\right)\bv_i + \sin\left(u_\theta(\bx)\right)-g(\bx) $ is the HTE-based residual prediction. Thus, we are only required to take an additional derivative to the HVP of the neural network model, which is much more efficient than that of the full Hessian.

The implementation details are given as follows. The model is a 4-layer fully connected network with 128 hidden units activated by Tanh, which is trained via Adam \cite{kingma2014adam} for 10K epochs with an initial learning rate 1e-3, which linearly decays to zero at the end of the optimization. We select 100 random residual points at each Adam epoch and 20K fixed testing points uniformly from the unit ball. We adopt the following model structure to satisfy the zero boundary condition with hard constraint and to avoid the boundary loss \cite{lu2021physics}
$
 (1 - \Vert\bx\Vert_2^2) u_\theta(\bx),
$
where $u_\theta(\bx)$ is the neural neural network model. We repeat our experiment five times with five independent random seeds and report the mean and standard deviation of the errors. For HTE, we choose $V=16$ in all settings. We choose $\lambda_{\text{gPINN}}$ as ten in 100D and 1K D cases and choose it as ten thousand in 10K D and 100K D cases. The gPINN loss weight $\lambda_{\text{gPINN}}$ can be chosen to make the scale of the PINN loss and the scale of the gPINN loss similar during initialization, i.e., choose $\lambda_{\text{gPINN}}$ such that ${r}_\theta(\bx)^2 \approx \lambda_{\text{gPINN}} \left\|\nabla_{\bx}{r}_\theta(\bx)\right\|^2$ and $\hat{r}_\theta(\bx)^2 \approx \lambda_{\text{gPINN}} \left\|\nabla_{\bx}\hat{r}_\theta(\bx)\right\|^2$ at initialization.

\begin{table}[htbp]
\centering
\begin{tabular}{|c|c|c|c|c|c|}
\hline
Method & Metric & 100D & 1K D & 10K D & 100K D \\ \hline\hline
\multirow{3}{*}{PINN \cite{raissi2019physics}} & Memory & 1123MB & 2915MB & \textgreater{}80GB & \textgreater{}80GB \\ \cline{2-6} 
 & Speed & 1367.61it/s & 141.45it/s & N.A. & N.A. \\ \cline{2-6} 
 & Error & 6.24E-3$\pm$2.83E-3 & 1.20E-3$\pm$6.47E-4 & N.A. & N.A. \\ \hline\hline
\multirow{3}{*}{gPINN \cite{yu2022gradient}} & Memory & 1123MB & 2915MB & \textgreater{}80GB & \textgreater{}80GB \\ \cline{2-6} 
 & Speed & 458.84it/s & 51.42it/s & N.A. & N.A. \\ \cline{2-6} 
 & Error & \textbf{4.20E-3$\pm$1.21E-2} & \textbf{4.42E-4$\pm$1.84E-4} & N.A. & N.A. \\ \hline\hline
\multirow{3}{*}{HTE PINN (Ours)} & Memory & 869MB & 890MB & 927MB & 1089MB \\ \cline{2-6} 
 & Speed & 1792.03it/s & 1566.19it/s & 1066.10it/s & 345.10it/s \\ \cline{2-6} 
 & Error & 6.30E-3$\pm$2.88E-3 & 1.25E-3$\pm$3.40E-4 & 1.84E-3$\pm$9.56E-4 & 2.38E-3$\pm$1.72E-3 \\ \hline\hline
\multirow{3}{*}{HTE gPINN (Ours)} & Memory & 869MB & 890MB & 927MB & 1089MB \\ \cline{2-6} 
 & Speed & 616.79it/s & 461.81it/s & 229.18it/s & 153.35it/s \\ \cline{2-6} 
 & Error & 5.08E-3$\pm$2.00E-3 & 4.88E-4$\pm$1.59E-4 & \textbf{4.68E-4$\pm$4.61E-4} & \textbf{4.77E-5$\pm$4.70E-5} \\ \hline
\end{tabular}
\caption{Computation results for gPINN.}
\label{tab:gpinn}
\end{table}

The computational results are presented in Table \ref{tab:gpinn}, where three algorithms are compared: vanilla PINN \cite{raissi2019physics}, vanilla gPINN \cite{yu2022gradient}, HTE-based PINN, and HTE-based gPINN. In the preliminary, we observe that both PINN and gPINN encounter failure when exceeding the 80GB memory limit of the A100 GPU in dimensions surpassing 1000D. Conversely, PINN and gPINN based on HTE successfully operate in dimensions as high as 100K, demonstrating swift performance. Comparing PINN and gPINN, we note an enhancement in the latter's performance improvement, albeit at a slightly slower convergence pace, due to the additional first-order derivatives of the residual. 
We also observe the same memory cost for PINN and gPINN since gPINN necessitates an additional gradient concerning the input $\bx$ for the PINN or HTE residual, and we accomplish this computation using forward-mode automatic differentiation. Forward mode is highly memory efficient, so it does not result in any additional memory consumption. Therefore, both PINN and gPINN have the same memory consumption. However, forward-mode computation does require more time.
The advantages of gPINN become more pronounced in higher dimensions (exceeding 1000D), leading to an additional order of magnitude reduction in PINN error. However, in scenarios with 100D and 1000D, HTE-based gPINN performs slightly inferior to conventional gPINN. This arises because HTE is ultimately an approximation, and in high-dimensional gPINN, it introduces $d$ additional losses, where $d$ is the dimensionality. Hence, the approximation with HTE tends to incur more significant errors. Nevertheless, the performance of HTE-based gPINN remains comparable to conventional gPINN. In summary, we verify that HTE can accelerate and scale up gPINN to very high dimensions. Moreover, gPINN continues to enhance the performance of PINN in high-dimensional PDEs. To the best of our knowledge, previous work \cite{yu2022gradient} only demonstrated gPINN's effectiveness in low-dimensional PDEs, while our work is the first to show that in high dimensions.

\subsection{Biharmonic Equation}
In this subsection, we further validate the proposed HTE method, demonstrating its capability to accelerate the convergence of PINNs in high-dimensional and high-order biharmonic equations while significantly reducing its memory consumption. The following complicated exact solution is considered to test the convergence of high-dimensional PDE solvers:
\begin{equation}
u_{\text{exact}}(\bx) = \left(1 - \Vert \bx \Vert_2^2\right)\left(4 - \Vert \bx \Vert_2^2\right)\left(\sum_{i=1}^{d-2}  c_i \exp(\bx_i\bx_{i+1}\bx_{i+2})\right),
\end{equation}
where $c_i \sim \mathcal{N}(0, 1)$.
The term $(1 - \Vert \bx \Vert_2^2)(4 - \Vert \bx \Vert_2^2)$ is multiplied for zero boundary condition to prevent the boundary loss from leaking the solution's information.

The Biharmonic equation is given as below
\begin{align}
&\Delta^2 u(\bx) = g(\bx), \quad \bx \in  \{\bx \ | \ 1 < \Vert \bx \Vert_2 < 2\},\\
&u(\bx) = 0, \quad \bx \in \{\bx \ |\ \Vert \bx \Vert_2 = 1\} \cup \{\bx \ |\ \Vert \bx \Vert_2 = 2\},
\end{align}
where $g(\bx) = \Delta^2 u_{\text{exact}}(\bx)$ is given by the exact solution.

Here are the implementation details. The model is a 4-layer fully connected network with 128 hidden units activated by Tanh, which is trained via Adam \cite{kingma2014adam} for 10K epochs in 100D, 150D, and 200D cases, and 20K epochs in the 50D case, with an initial learning rate 1e-3, which linearly decays to zero at the end of the optimization. We select 100 random residual points at each Adam epoch and 20K fixed testing points within the domain $\bx \in  \{\bx \ | \ 1 \leq \Vert \bx \Vert_2 \leq 2\}$. We adopt the following model structure to satisfy the zero boundary condition with hard constraint and to avoid the boundary loss \cite{lu2021physics}:
$
 (1 - \Vert\bx\Vert_2^2) (4 - \Vert\bx\Vert_2^2) u_\theta(\bx),
$
where $u_\theta(\bx)$ is the neural neural network model. We repeat our experiment five times with five independent random seeds and report the mean and standard deviation of the errors.
For HTE, we control the batch size $V=16/512/1024$ to see how $V$ affects convergence.

\begin{table}[htbp]
\centering
\begin{tabular}{|c|c|c|c|c|c|}
\hline
Method & Metric & 50D & 100D & 150D & 200D \\ \hline\hline
\multirow{3}{*}{PINN \cite{raissi2019physics}} & Speed  & 19.82it/s & 6.56it/s & 3.50it/s  & N.A.\\ \cline{2-6} 
 & Memory  & 6199MB & 21453MB & 44631MB & $>$80GB \\ \cline{2-6} 
 & Error & 2.34E-2$\pm$1.91E-2 & 6.01E-3$\pm$2.32E-3 & 3.11E-3$\pm$1.79E-3 & N.A. \\ \hline\hline
 \multirow{3}{*}{HTE ($V=16$)} & Speed  & 438.79it/s	&109.57it/s	&54.86it/s&	25.89it/s
\\ \cline{2-6} 
& Memory & 991MB	&2085MB	&3589MB&	7225MB
\\ \cline{2-6} 
 & Error & 2.36E-2$\pm$1.97E-2 & 6.83E-3$\pm$1.63E-3 & 3.91E-3$\pm$1.47E-3 & 4.40E-3$\pm$5.65E-3\\ \hline\hline
\multirow{3}{*}{HTE ($V=512$)} & Speed  & 123.98it/s  & 54.07it/s& 42.44it/s & 23.50it/s\\ \cline{2-6} 
& Memory & 1861MB & 2089MB & 3593MB & 7229MB\\ \cline{2-6} 
 & Error & 2.34E-2$\pm$1.96E-2 & 6.08E-3$\pm$2.28E-3 & 3.39E-3$\pm$1.99E-3 & 4.03E-3$\pm$4.85E-3\\ \hline\hline
 \multirow{3}{*}{HTE ($V=1024$)} & Speed  &78.50it/s	&34.72it/s&	33.77it/s&	20.36it/s\\ \cline{2-6} 
& Memory & 2535MB&	3037MB	&3595MB&	7231MB
\\ \cline{2-6} 
 & Error & 2.34E-2$\pm$1.96E-2 & 6.01E-3$\pm$2.35E-3 & 3.13E-3$\pm$1.82E-3 & 3.97E-3$\pm$4.91E-3\\ \hline
\end{tabular}
\caption{Computational results for the Biharmonic equation. In all-dimensional biharmonic PDEs, HTE consistently achieves results similar to those of full PINN while significantly outpacing full PINN in terms of speed and notably reducing memory consumption. In the case of 200 dimensions, where full PINN exceeds memory limits, HTE still converges rapidly.}
\label{tab:biharmonic}
\end{table}

The computational results for the biharmonic equation are presented in Table \ref{tab:biharmonic}, where we report the speed, memory, and relative $L_2$ error of the baseline PINN and our proposed HTE method. Regarding the regular PINN, we observe a significant slowdown in speed and a rapid increase in memory consumption as the dimensionality rises. This is because the regular PINNs require the computation of the full biharmonic operator, and the computational cost increases to the fourth power with respect to the dimension. As a result, regular PINNs exceed the memory limit of an A100 GPU in the case of 200 dimensions. In contrast, the second-order PDE case requires 10K dimensions for PINN to go out-of-memory, signifying that the biharmonic operator is much more costly than the Hessian due to the high-order derivative.
For HTE, we opted for a batch size of $V=512/1024$ to make HTE's convergence results similar to those of PINN, which is in contrast to our earlier example with a second-order PDE, where we used a batch size of $V=16$. If we just use a small batch size $V=16$ in the biharmonic equation, then HTE will perform worse than PINN. The reason behind this is our variance analysis, indicating that when estimating the biharmonic operator, both diagonal and off-diagonal elements contribute to the variance, resulting in relatively larger errors than estimating the Hessian trace. Nevertheless, HTE demonstrates substantial improvements in terms of speed and memory consumption. In scenarios exceeding 50 dimensions, its speed is nearly ten times that of regular PINN, and the rate of memory consumption growth is significantly lower than that of PINN.

\section{Summary}
In this paper, we introduced the Hutchinson Trace Estimation (HTE) method to enhance the range of capabilities of Physics-Informed Neural Networks (PINNs) in tackling high-dimensional and high-order partial differential equations (PDEs). While PINNs have succeeded in low-dimensional problems, their application to complex PDEs has faced challenges due to computational bottlenecks in automatic differentiation. HTE, which is applied to calculate Hessian vector products (HVPs), has emerged as an effective solution to accelerate PINNs and reduce memory costs.

We first validated our approach on second-order parabolic equations, demonstrating the convergence of the PINN loss with HTE to the original PINNs loss under certain conditions, guaranteeing the effectiveness of the HTE approximation. Comparisons with Stochastic Dimension Gradient Descent (SDGD) highlighted the superior performance of HTE in scenarios with significant dimensionality variability and variance, where we delved into the theoretical comparison between the two methods by proving their corresponding variance. The extension of HTE to the biharmonic equation showcased its efficiency in handling higher-order, high-dimensional PDEs by employing tensor-vector products (TVPs) to transform computations.

Experimental setups confirmed the effectiveness of HTE, revealing comparable convergence effects with SDGD under memory and speed constraints. Furthermore, HTE proved valuable in accelerating Gradient-Enhanced PINN (gPINN), which requires an additional first-order derivative over the PINNs residual and the fourth-order high dimensional biharmonic equation for better computational efficiency.

In conclusion, our proposed HTE method addresses the limitations of PINNs in handling high-dimensional and high-order PDEs, providing a versatile and efficient approach. Our work contributes to the scientific machine learning field by presenting a unified framework for solving general high-dimensional PDEs, demonstrating the effectiveness of HTE, and discussing its advantages over existing methods.

\section*{Acknowledgement}
This research of ZH, ZS, and KK is partially supported by the National Research Foundation Singapore under the AI Singapore Programme (AISG Award No: AISG2-TC-2023-010-SGIL) and the Singapore Ministry of Education Academic Research Fund Tier 1 (Award No: T1 251RES2207).
The work of GEK was supported by the MURI-AFOSR FA9550-20-1-0358 projects and by the DOE SEA-CROGS project (DE-SC0023191). GEK was also supported by the ONR Vannevar Bush Faculty Fellowship (N00014-22-1-2795).

\newpage

\appendix

\section{Proof}
\subsection{Proof of Theorem \ref{thm:unbiased}}\label{appendix:unbiased}
\begin{theorem} (Theorem \ref{thm:unbiased}) in the main text)
The loss $L_{\text{HTE}}(\theta)$ in equation (\ref{eq:HTE}) converges almost surely (a.s.) to the exact PINN loss $L_{\text{PINN}}(\theta)$ in equation (\ref{eq:PINN}), as $V \rightarrow \infty$, i.e.,
\begin{align}
\mathbb{P}\left(\lim_{V \rightarrow \infty}L_{\text{HTE}}(\theta;\{\bv_i\}_{i=1}^V) = L_{\text{PINN}}(\theta)\right) = 1.
\end{align}
The loss $L_{\text{HTE, unbiased}}(\theta)$ in equation (\ref{eq:HTE_unbiased}) is an unbiased estimator for the exact PINN loss $L_{\text{PINN}}(\theta)$ in equation (\ref{eq:PINN}), i.e.,
\begin{align}
\mathbb{E}_{\{\bv_i, \hat{\bv}_i\}_{i=1}^V}\left[L_{\text{HTE, unbiased}}(\theta;\{\bv_i,\hat{\bv}_i\}_{i=1}^V)\right] = L_{\text{PINN}}(\theta).
\end{align}
\end{theorem}

\begin{proof}
\textbf{Convergence of $L_{\text{HTE}}(\theta; \{\bv_i\}_{i=1}^V)$}. Due to the strong law of large numbers and the fact that HTE is unbiased, i.e., $\mathbb{E}_{\{\bv_i\}_{i=1}^V} \left[\frac{1}{V}\sum_{i=1}^V\bv_i^\mathrm{T}A_\theta(\bx,t)\bv_i \right] = \operatorname{Tr}\left(A_{\theta}(\bx, t)\right)$,
\begin{equation}
\mathbb{P}\left(\lim_{V \rightarrow \infty}\frac{1}{V}\sum_{i=1}^V\bv_i^\mathrm{T}A_\theta(\bx,t)\bv_i  = \operatorname{Tr}\left(A_{\theta}(\bx, t)\right)\right) = 1.
\end{equation}
Thus,
\begin{equation}
\mathbb{P}\left(\lim_{V \rightarrow \infty}\left(\frac{1}{V}\sum_{i=1}^V\bv_i^\mathrm{T}A_\theta(\bx,t)\bv_i + B_{\theta}(\bx, t)\right)^2 = \left(\operatorname{Tr}\left(A_{\theta}(\bx, t)\right)+ B_{\theta}(\bx, t)\right)^2\right) = 1.
\end{equation}

\textbf{Unbiasedness of $L_{\text{HTE, unbiased}}(\theta; \{\bv_i, \hat{\bv}_i\}_{i=1}^V)$}. Since HTE is an unbiased trace estimator, i.e.,
\begin{align}
&\quad\mathbb{E}_{\{\bv_i\}_{i=1}^V}\left[\frac{1}{V}\sum_{i=1}^V\bv_i^\mathrm{T}A_\theta(\bx,t)\bv_i + B_{\theta}(\bx, t)\right] \\
&= \mathbb{E}_{\{\hat{\bv}_i\}_{i=1}^V}\left[\frac{1}{V}\sum_{i=1}^V\hat{\bv}_i^\mathrm{T}A_\theta(\bx,t)\hat{\bv}_i + B_{\theta}(\bx, t)\right] = \operatorname{Tr}\left(A_{\theta}(\bx, t)\right) + B_{\theta}(\bx, t).
\end{align}
The random variables $\{\bv_i\}_{i=1}^V$ and $\{\hat{\bv}_i\}_{i=1}^V$ are also independent. Thus
\begin{equation}
\begin{aligned}
&\quad\mathbb{E}_{\{\bv_i,\hat{\bv}_i\}_{i=1}^V}\left[L_{\text{HTE, unbiased}}(\theta;\{\bv_i,\hat{\bv}_i\}_{i=1}^V)\right] \\
&= \frac{1}{2}\mathbb{E}_{\{\bv_i\}_{i=1}^V}\left[\frac{1}{V}\sum_{i=1}^V\bv_i^\mathrm{T}A_\theta(\bx,t)\bv_i + B_{\theta}(\bx, t)\right]\mathbb{E}_{\{\hat{\bv}_i\}_{i=1}^V}\left[\frac{1}{V}\sum_{i=1}^V\hat{\bv}_i^\mathrm{T}A_\theta(\bx,t)\hat{\bv}_i + B_{\theta}(\bx, t)\right]\\
&= \frac{1}{2}\left[\operatorname{Tr}\left(A_{\theta}(\bx, t)\right) + B_{\theta}(\bx, t)\right]^2\\
&= L_{\text{PINN}}(\theta).
\end{aligned}
\end{equation}
\end{proof}

\subsection{Proof of Theorem \ref{thm:sdgd_var}}\label{appendix:sdgd_var}

\begin{theorem} (Theorem \ref{thm:sdgd_var} in the main text)
Given the SDGD estimator of the trace 
$
\operatorname{Tr}(A) \approx \frac{d}{B}\sum_{i\in I}A_{ii}
$
given an index set $I \subset \{1,2,\cdots,d\}$ whose cardinality $|I| = B$, where $B$ is the SDGD's batch size for the dimension, its variance is
\begin{equation}
\begin{aligned}
\mathbb{V}\left[\frac{d}{B}\sum_{i\in I}A_{ii}\right] = \frac{1}{\binom{d}{B}}\left(\sum_{I: |I| = B}\frac{d}{B}\sum_{i\in I}A_{ii} - \sum_{i=1}^dA_{ii}\right)^2,
\end{aligned}
\end{equation}
where $\binom{d}{B} = \frac{d(d-1)\cdots(d-B+1)}{B(B-1)\cdots1}$.
\end{theorem}

\begin{proof}
SDGD's trace estimator is
$
\operatorname{Tr}(A) \approx \frac{d}{|I|}\sum_{i\in I}A_{ii}.
$
Suppose that the batch size of the sampled dimension in SDGD is $B$; then the variance can be directly computed:
\begin{equation}
\begin{aligned}
\mathbb{V}\left[\frac{d}{B}\sum_{i\in I}A_{ii}\right] = \frac{1}{\binom{d}{B}}\left(\sum_{I: |I| = B}\frac{d}{B}\sum_{i\in I}A_{ii} - \sum_{i=1}^dA_{ii}\right)^2,
\end{aligned}
\end{equation}
where $\binom{d}{B} = \frac{d(d-1)\cdots(d-B+1)}{B(B-1)\cdots1}$.
\end{proof}

\subsection{Proof of Theorem \ref{thm:hte_var}}\label{appendix:hte_var}
\begin{theorem} (Theorem \ref{thm:hte_var} in the main text)
Given the HTE estimator of the trace 
$
\operatorname{Tr}(A) \approx\frac{1}{V}\sum_{k=1}^V \bv_k^\mathrm{T} A\bv_k,
$ 
where $V$ is the HTE batch size and each dimension of $\bv_k \in \mathbb{R}^d$ is an $i.i.d.$ sample from the Rademacher distribution, then its variance is
$
\frac{1}{V} \sum_{i \neq j}A_{ij}^2.
$
\end{theorem}
\begin{proof}
The estimator of HTE is
$
\operatorname{Tr}(A) \approx \frac{1}{V}\sum_{k=1}^V \bv_k^\mathrm{T} A\bv_k = \sum_{i=1}^d A_{ii} + \frac{1}{V}\sum_{i \neq j}A_{ij}\sum_{k=1}^V\bv_{k,i}\bv_{k, j},
$
where $\bv_{k,i}$ denotes the $i$th dimension of the vector $\bv_k \in \mathbb{R}^d$ and we are using $p(\bv)$ as the Rademacher distribution.

HTE's variance can be computed as follows:
\begin{equation}
\begin{aligned}
\left( \bv^\mathrm{T} A\bv - \operatorname{Tr}(A)\right)^2&= \left(\sum_{i \neq j}A_{ij}\bv_{i}\bv_{j}\right)^2\\
&= \sum_{i \neq j}\sum_{k \neq l}A_{ij}A_{kl}\bv_{i}\bv_{j}\bv_k\bv_l.
\end{aligned}
\end{equation}
If $i = k$ and $j = l$, then $\mathbb{E}[\bv_{i}\bv_{j}\bv_k\bv_l] = \mathbb{E}[v_i^2]\mathbb{E}[v_j^2] = 1$, else equals zero. Therefore, 
\begin{equation}
\begin{aligned}
\mathbb{V}\left[\bv^\mathrm{T}A\bv\right] &= \mathbb{E}\left[\left( \bv^\mathrm{T} A\bv - \operatorname{Tr}(A)\right)^2\right]= \sum_{i \neq j}A_{ij}^2.
\end{aligned}
\end{equation}
With the HTE estimator of batch size $V$ and $\bv_k$ are independent for different $k$, the variance is
\begin{equation}
\mathbb{V}\left[\frac{1}{V}\sum_{k=1}^V \bv_k^\mathrm{T} A\bv_k\right] = \frac{1}{V^2}\sum_{k=1}^V\mathbb{V}\left[ \bv_k^\mathrm{T} A\bv_k\right] = \frac{1}{V^2}\sum_{k=1}^V\left[ \sum_{i \neq j}A_{ij}^2\right]= \frac{1}{V} \sum_{i \neq j}A_{ij}^2.
\end{equation}
\end{proof}

\subsection{Proof of Theorem \ref{thm:biharmonic}}\label{appendix:biharmonic}
\begin{theorem}(Theorem \ref{thm:biharmonic} in the main text)
The biharmonic operator can be unbiasedly estimated using the following TVPs:
\begin{equation}
\begin{aligned}
\Delta^2 u(\bx) &= \frac{1}{3}\mathbb{E}_{\bv\sim \mathcal{N}(0, I)}\left[\frac{\partial^4}{\partial\bx^4}u(\bx)[\bv, \bv, \bv, \bv] \right],
\end{aligned}
\end{equation}
where $\mathcal{N}(0, I)$ is the $d$-dimensional unit Gaussian, and
\begin{equation}
\frac{\partial^4}{\partial\bx^4}u(\bx) \in \mathbb{R}^{d \times d \times d \times d}, \quad \text{where }\left[\frac{\partial^4}{\partial\bx^4}u(\bx)\right]_{ijkl} = \frac{\partial^4}{\partial\bx_i\partial\bx_j\partial\bx_k\partial\bx_l}u(\bx),
\end{equation}
and the TVP is defined as 
\begin{equation}
\frac{\partial^4}{\partial\bx^4}u(\bx)[\bv, \bv, \bv, \bv] = \sum_{i,j,k,l=1}^d\frac{\partial^4}{\partial\bx_i\partial\bx_j\partial\bx_k\partial\bx_l}u(\bx)\bv_i\bv_j\bv_k\bv_l,
\end{equation}
\end{theorem}
\begin{proof}
For the unit Gaussian distribution, we have that 
\begin{equation}
\begin{aligned}
\mathbb{E}_{\bv\sim \mathcal{N}(0,I)}\left[\frac{\partial^4}{\partial\bx^4}u(\bx)[\bv, \bv, \bv, \bv] \right] &= \mathbb{E}_{\bv\sim \mathcal{N}(0,I)}\left[\sum_{i,j,k,l=1}^d\frac{\partial^4}{\partial\bx_i\partial\bx_j\partial\bx_k\partial\bx_l}u(\bx)\bv_i\bv_j\bv_k\bv_l\right]\\
&= \mathbb{E}_{\bv\sim \mathcal{N}(0,I)}\left[\sum_{i=1}^d\frac{\partial^4}{\partial\bx_i^4}u(\bx)\bv_i^4\right] + 6\mathbb{E}_{\bv\sim \mathcal{N}(0,I)}\left[\sum_{i \neq j}\frac{\partial^4}{\partial\bx_i^2\partial \bx_j^2}u(\bx)\bv_i^2\bv_j^2\right] \\
&= 3\sum_{i=1}^d\frac{\partial^4}{\partial\bx_i^4}u(\bx) + 6 \sum_{i \neq j}\frac{\partial^4}{\partial\bx_i^2\partial\bx_j^2}u(\bx)\\
&= 3\Delta^2u(\bx).
\end{aligned}
\end{equation}
Here, we use the fact that for the one-dimensional Gaussian random variable $v \sim \mathcal{N}(0, 1)$, its second moment $\mathbb{E}_v [v^2] = 1$ and its fourth moment $\mathbb{E}_v [v^4] = 3$. Note that each entry of the vector $\bv$ is the one-dimensional Gaussian.
\end{proof}

\newpage

\bibliographystyle{plain}
\bibliography{main}

\begin{thebibliography}{10}

\bibitem{beck2021deep}
Christian Beck, Sebastian Becker, Patrick Cheridito, Arnulf Jentzen, and Ariel Neufeld.
\newblock Deep splitting method for parabolic pdes.
\newblock {\em SIAM Journal on Scientific Computing}, 43(5):A3135--A3154, 2021.

\bibitem{beck2019machine}
Christian Beck, Weinan E, and Arnulf Jentzen.
\newblock Machine learning approximation algorithms for high-dimensional fully nonlinear partial differential equations and second-order backward stochastic differential equations.
\newblock {\em Journal of Nonlinear Science}, 29:1563--1619, 2019.

\bibitem{beck2020overcoming}
Christian Beck, Lukas Gonon, and Arnulf Jentzen.
\newblock Overcoming the curse of dimensionality in the numerical approximation of high-dimensional semilinear elliptic partial differential equations.
\newblock {\em arXiv preprint arXiv:2003.00596}, 2020.

\bibitem{beck2020overcoming_ac}
Christian Beck, Fabian Hornung, Martin Hutzenthaler, Arnulf Jentzen, and Thomas Kruse.
\newblock Overcoming the curse of dimensionality in the numerical approximation of allen--cahn partial differential equations via truncated full-history recursive multilevel picard approximations.
\newblock {\em Journal of Numerical Mathematics}, 28(4):197--222, 2020.

\bibitem{becker2020numerical}
Sebastian Becker, Ramon Braunwarth, Martin Hutzenthaler, Arnulf Jentzen, and Philippe von Wurstemberger.
\newblock Numerical simulations for full history recursive multilevel picard approximations for systems of high-dimensional partial differential equations.
\newblock {\em arXiv preprint arXiv:2005.10206}, 2020.

\bibitem{becker2021solving}
Sebastian Becker, Patrick Cheridito, Arnulf Jentzen, and Timo Welti.
\newblock Solving high-dimensional optimal stopping problems using deep learning.
\newblock {\em European Journal of Applied Mathematics}, 32(3):470--514, 2021.

\bibitem{bettencourt2019taylormode}
Jesse Bettencourt, Matthew~J. Johnson, and David Duvenaud.
\newblock Taylor-mode automatic differentiation for higher-order derivatives in {JAX}.
\newblock In {\em Program Transformations for ML Workshop at NeurIPS 2019}, 2019.

\bibitem{jax2018github}
James Bradbury, Roy Frostig, Peter Hawkins, Matthew~James Johnson, Chris Leary, Dougal Maclaurin, George Necula, Adam Paszke, Jake Vander{P}las, Skye Wanderman-{M}ilne, and Qiao Zhang.
\newblock {JAX}: composable transformations of {P}ython+{N}um{P}y programs, 2018.

\bibitem{cai2021physics}
Shengze Cai, Zhiping Mao, Zhicheng Wang, Minglang Yin, and George~Em Karniadakis.
\newblock Physics-informed neural networks (pinns) for fluid mechanics: A review.
\newblock {\em Acta Mechanica Sinica}, 37(12):1727--1738, 2021.

\bibitem{chan2019machine}
Quentin Chan-Wai-Nam, Joseph Mikael, and Xavier Warin.
\newblock Machine learning for semi linear pdes.
\newblock {\em Journal of scientific computing}, 79(3):1667--1712, 2019.

\bibitem{chen2021solving}
Xiaoli Chen, Liu Yang, Jinqiao Duan, and George~Em Karniadakis.
\newblock Solving inverse stochastic problems from discrete particle observations using the fokker--planck equation and physics-informed neural networks.
\newblock {\em SIAM Journal on Scientific Computing}, 43(3):B811--B830, 2021.

\bibitem{cho2022separable}
Junwoo Cho, Seungtae Nam, Hyunmo Yang, Seok-Bae Yun, Youngjoon Hong, and Eunbyung Park.
\newblock Separable pinn: Mitigating the curse of dimensionality in physics-informed neural networks.
\newblock {\em arXiv preprint arXiv:2211.08761}, 2022.

\bibitem{darbon2020overcoming}
J{\'e}r{\^o}me Darbon, Gabriel~P Langlois, and Tingwei Meng.
\newblock Overcoming the curse of dimensionality for some hamilton--jacobi partial differential equations via neural network architectures.
\newblock {\em Research in the Mathematical Sciences}, 7:1--50, 2020.

\bibitem{darbon2016algorithms}
J{\'e}r{\^o}me Darbon and Stanley Osher.
\newblock Algorithms for overcoming the curse of dimensionality for certain hamilton--jacobi equations arising in control theory and elsewhere.
\newblock {\em Research in the Mathematical Sciences}, 3(1):19, 2016.

\bibitem{haghighat2021physics}
Ehsan Haghighat, Maziar Raissi, Adrian Moure, Hector Gomez, and Ruben Juanes.
\newblock A physics-informed deep learning framework for inversion and surrogate modeling in solid mechanics.
\newblock {\em Computer Methods in Applied Mechanics and Engineering}, 379:113741, 2021.

\bibitem{han2018solving}
Jiequn Han, Arnulf Jentzen, and Weinan E.
\newblock Solving high-dimensional partial differential equations using deep learning.
\newblock {\em Proceedings of the National Academy of Sciences}, 115(34):8505--8510, 2018.

\bibitem{han2017deep}
Jiequn Han, Arnulf Jentzen, et~al.
\newblock Deep learning-based numerical methods for high-dimensional parabolic partial differential equations and backward stochastic differential equations.
\newblock {\em Communications in mathematics and statistics}, 5(4):349--380, 2017.

\bibitem{he2023learning}
Di~He, Shanda Li, Wenlei Shi, Xiaotian Gao, Jia Zhang, Jiang Bian, Liwei Wang, and Tie-Yan Liu.
\newblock Learning physics-informed neural networks without stacked back-propagation.
\newblock In {\em International Conference on Artificial Intelligence and Statistics}, pages 3034--3047. PMLR, 2023.

\bibitem{henry2017deep}
Pierre Henry-Labordere.
\newblock Deep primal-dual algorithm for bsdes: Applications of machine learning to cva and im.
\newblock {\em Available at SSRN 3071506}, 2017.

\bibitem{hu2022augmented}
Zheyuan Hu, Ameya~D Jagtap, George~Em Karniadakis, and Kenji Kawaguchi.
\newblock Augmented physics-informed neural networks (apinns): A gating network-based soft domain decomposition methodology.
\newblock {\em arXiv preprint arXiv:2211.08939}, 2022.

\bibitem{hu2021extended}
Zheyuan Hu, Ameya~D. Jagtap, George~Em Karniadakis, and Kenji Kawaguchi.
\newblock When do extended physics-informed neural networks (xpinns) improve generalization?
\newblock {\em SIAM Journal on Scientific Computing}, 44(5):A3158--A3182, 2022.

\bibitem{hu2023tackling}
Zheyuan Hu, Khemraj Shukla, George~Em Karniadakis, and Kenji Kawaguchi.
\newblock Tackling the curse of dimensionality with physics-informed neural networks.
\newblock {\em arXiv preprint arXiv:2307.12306}, 2023.

\bibitem{hu2023bias}
Zheyuan Hu, Zhouhao Yang, Yezhen Wang, George~Em Karniadakis, and Kenji Kawaguchi.
\newblock Bias-variance trade-off in physics-informed neural networks with randomized smoothing for high-dimensional pdes.
\newblock {\em arXiv preprint arXiv:2311.15283}, 2023.

\bibitem{hu2024score}
Zheyuan Hu, Zhongqiang Zhang, George~Em Karniadakis, and Kenji Kawaguchi.
\newblock Score-based physics-informed neural networks for high-dimensional fokker-planck equations.
\newblock {\em arXiv preprint arXiv:2402.07465}, 2024.

\bibitem{hure2020deep}
C{\^o}me Hur{\'e}, Huy{\^e}n Pham, and Xavier Warin.
\newblock Deep backward schemes for high-dimensional nonlinear pdes.
\newblock {\em Mathematics of Computation}, 89(324):1547--1579, 2020.

\bibitem{hutchinson1989stochastic}
Michael~F Hutchinson.
\newblock A stochastic estimator of the trace of the influence matrix for laplacian smoothing splines.
\newblock {\em Communications in Statistics-Simulation and Computation}, 18(3):1059--1076, 1989.

\bibitem{hutzenthaler2020overcoming}
Martin Hutzenthaler, Arnulf Jentzen, Thomas Kruse, Tuan Anh~Nguyen, and Philippe von Wurstemberger.
\newblock Overcoming the curse of dimensionality in the numerical approximation of semilinear parabolic partial differential equations.
\newblock {\em Proceedings of the Royal Society A}, 476(2244):20190630, 2020.

\bibitem{hutzenthaler2021multilevel}
Martin Hutzenthaler, Arnulf Jentzen, Thomas Kruse, et~al.
\newblock Multilevel picard iterations for solving smooth semilinear parabolic heat equations.
\newblock {\em Partial Differential Equations and Applications}, 2(6):1--31, 2021.

\bibitem{jagtap2020extended}
Ameya~D Jagtap and George~Em Karniadakis.
\newblock Extended physics-informed neural networks (xpinns): A generalized space-time domain decomposition based deep learning framework for nonlinear partial differential equations.
\newblock {\em Communications in Computational Physics}, 28(5):2002--2041, 2020.

\bibitem{jagtap2020adaptive}
Ameya~D Jagtap, Kenji Kawaguchi, and George~Em Karniadakis.
\newblock Adaptive activation functions accelerate convergence in deep and physics-informed neural networks.
\newblock {\em Journal of Computational Physics}, 404:109136, 2020.

\bibitem{jagtap2022deep}
Ameya~D Jagtap, Dimitrios Mitsotakis, and George~Em Karniadakis.
\newblock Deep learning of inverse water waves problems using multi-fidelity data: Application to serre--green--naghdi equations.
\newblock {\em Ocean Engineering}, 248:110775, 2022.

\bibitem{ji2020three}
Shaolin Ji, Shige Peng, Ying Peng, and Xichuan Zhang.
\newblock Three algorithms for solving high-dimensional fully coupled fbsdes through deep learning.
\newblock {\em IEEE Intelligent Systems}, 35(3):71--84, 2020.

\bibitem{jin2021nsfnets}
Xiaowei Jin, Shengze Cai, Hui Li, and George~Em Karniadakis.
\newblock Nsfnets (navier-stokes flow nets): Physics-informed neural networks for the incompressible navier-stokes equations.
\newblock {\em Journal of Computational Physics}, 426:109951, 2021.

\bibitem{karniadakis2021physics}
George~Em Karniadakis, Ioannis~G Kevrekidis, Lu~Lu, Paris Perdikaris, Sifan Wang, and Liu Yang.
\newblock Physics-informed machine learning.
\newblock {\em Nature Reviews Physics}, 3(6):422--440, 2021.

\bibitem{kawaguchi2016deep}
Kenji Kawaguchi.
\newblock Deep learning without poor local minima.
\newblock In {\em Advances in neural information processing systems (NeurIPS)}, pages 586--594, 2016.

\bibitem{kawaguchi2023does}
Kenji Kawaguchi, Zhun Deng, Xu~Ji, and Jiaoyang Huang.
\newblock How does information bottleneck help deep learning?
\newblock {\em arXiv preprint arXiv:2305.18887}, 2023.

\bibitem{kawaguchi2017generalization}
Kenji Kawaguchi, Leslie~Pack Kaelbling, and Yoshua Bengio.
\newblock Generalization in deep learning.
\newblock {\em Cambridge University Press}, 2022.

\bibitem{kingma2014adam}
Diederik~P Kingma and Jimmy Ba.
\newblock Adam: A method for stochastic optimization.
\newblock {\em ICLR}, 2015.

\bibitem{lu2021physics}
Lu~Lu, Raphael Pestourie, Wenjie Yao, Zhicheng Wang, Francesc Verdugo, and Steven~G Johnson.
\newblock Physics-informed neural networks with hard constraints for inverse design.
\newblock {\em SIAM Journal on Scientific Computing}, 43(6):B1105--B1132, 2021.

\bibitem{meyer2021hutch++}
Raphael~A Meyer, Cameron Musco, Christopher Musco, and David~P Woodruff.
\newblock Hutch++: Optimal stochastic trace estimation.
\newblock In {\em Symposium on Simplicity in Algorithms (SOSA)}, pages 142--155. SIAM, 2021.

\bibitem{mishra2020estimates}
Siddhartha Mishra and Roberto Molinaro.
\newblock Estimates on the generalization error of physics informed neural networks (pinns) for approximating pdes.
\newblock {\em arXiv preprint arXiv:2006.16144}, 2020.

\bibitem{oktay2021randomized}
Deniz Oktay, Nick McGreivy, Joshua Aduol, Alex Beatson, and Ryan~P Adams.
\newblock Randomized automatic differentiation.
\newblock In {\em International Conference on Learning Representations}, 2021.

\bibitem{persson2022improved}
David Persson, Alice Cortinovis, and Daniel Kressner.
\newblock Improved variants of the hutch++ algorithm for trace estimation.
\newblock {\em SIAM Journal on Matrix Analysis and Applications}, 43(3):1162--1185, 2022.

\bibitem{psaros2022meta}
Apostolos~F Psaros, Kenji Kawaguchi, and George~Em Karniadakis.
\newblock Meta-learning pinn loss functions.
\newblock {\em Journal of computational physics}, 458:111121, 2022.

\bibitem{raissi2018forward}
Maziar Raissi.
\newblock Forward-backward stochastic neural networks: Deep learning of high-dimensional partial differential equations.
\newblock {\em arXiv preprint arXiv:1804.07010}, 2018.

\bibitem{raissi2019physics}
Maziar Raissi, Paris Perdikaris, and George~E Karniadakis.
\newblock Physics-informed neural networks: A deep learning framework for solving forward and inverse problems involving nonlinear partial differential equations.
\newblock {\em Journal of Computational Physics}, 378:686--707, 2019.

\bibitem{roosta2015improved}
Farbod Roosta-Khorasani and Uri Ascher.
\newblock Improved bounds on sample size for implicit matrix trace estimators.
\newblock {\em Foundations of Computational Mathematics}, 15(5):1187--1212, 2015.

\bibitem{shin2020convergence}
Yeonjong Shin, Jerome Darbon, and George~Em Karniadakis.
\newblock On the convergence of physics informed neural networks for linear second-order elliptic and parabolic type pdes.
\newblock {\em arXiv preprint arXiv:2004.01806}, 2020.

\bibitem{sirignano2018dgm}
Justin Sirignano and Konstantinos Spiliopoulos.
\newblock Dgm: A deep learning algorithm for solving partial differential equations.
\newblock {\em Journal of computational physics}, 375:1339--1364, 2018.

\bibitem{skorski2021modern}
Maciej Skorski.
\newblock Modern analysis of hutchinson's trace estimator.
\newblock In {\em 2021 55th Annual Conference on Information Sciences and Systems (CISS)}, pages 1--5. IEEE, 2021.

\bibitem{song2021scorebased}
Yang Song, Jascha Sohl-Dickstein, Diederik~P Kingma, Abhishek Kumar, Stefano Ermon, and Ben Poole.
\newblock Score-based generative modeling through stochastic differential equations.
\newblock In {\em International Conference on Learning Representations}, 2021.

\bibitem{vaswani2017attention}
Ashish Vaswani, Noam Shazeer, Niki Parmar, Jakob Uszkoreit, Llion Jones, Aidan~N Gomez, {\L}ukasz Kaiser, and Illia Polosukhin.
\newblock Attention is all you need.
\newblock {\em Advances in neural information processing systems}, 30, 2017.

\bibitem{wang20222}
Chuwei Wang, Shanda Li, Di~He, and Liwei Wang.
\newblock Is \$l{\textasciicircum}2\$ physics informed loss always suitable for training physics informed neural network?
\newblock In Alice~H. Oh, Alekh Agarwal, Danielle Belgrave, and Kyunghyun Cho, editors, {\em Advances in Neural Information Processing Systems}, 2022.

\bibitem{wang2022tensor}
Yifan Wang, Pengzhan Jin, and Hehu Xie.
\newblock Tensor neural network and its numerical integration.
\newblock {\em arXiv preprint arXiv:2207.02754}, 2022.

\bibitem{wang2022solving}
Yifan Wang, Yangfei Liao, and Hehu Xie.
\newblock Solving schr$\backslash$"$\{$o$\}$ dinger equation using tensor neural network.
\newblock {\em arXiv preprint arXiv:2209.12572}, 2022.

\bibitem{Weinan2017TheDR}
E.~Weinan and Ting Yu.
\newblock The deep ritz method: A deep learning-based numerical algorithm for solving variational problems.
\newblock {\em Communications in Mathematics and Statistics}, 6:1--12, 2017.

\bibitem{yang2019adversarial}
Yibo Yang and Paris Perdikaris.
\newblock Adversarial uncertainty quantification in physics-informed neural networks.
\newblock {\em Journal of Computational Physics}, 394:136--152, 2019.

\bibitem{yu2022gradient}
Jeremy Yu, Lu~Lu, Xuhui Meng, and George~Em Karniadakis.
\newblock Gradient-enhanced physics-informed neural networks for forward and inverse pde problems.
\newblock {\em Computer Methods in Applied Mechanics and Engineering}, 393:114823, 2022.

\bibitem{zang2020weak}
Yaohua Zang, Gang Bao, Xiaojing Ye, and Haomin Zhou.
\newblock Weak adversarial networks for high-dimensional partial differential equations.
\newblock {\em Journal of Computational Physics}, 411:109409, 2020.

\end{thebibliography}

\end{document}